\documentclass[authoryear,12pt]{elsarticle}
\usepackage{amsmath,amssymb,amsthm}
\usepackage[ruled]{algorithm2e}
\usepackage{graphicx}
\usepackage{color}
\usepackage{xcolor}
\usepackage{booktabs}
\usepackage{caption}
\usepackage{subcaption}
\usepackage{soul}
\usepackage{xurl}
\graphicspath{{fig/}}
\usepackage{algorithmic}

\usepackage{setspace}
\usepackage[T1]{fontenc}				
\usepackage{placeins}
\usepackage{ragged2e}
\usepackage{bm}                         
\usepackage{array}
\usepackage{listings}
\usepackage{xcolor}
\definecolor{codegreen}{rgb}{0,0.6,0}
\definecolor{codegray}{rgb}{0.5,0.5,0.5}
\definecolor{codepurple}{rgb}{0.58,0,0.82}
\definecolor{backcolour}{rgb}{0.95,0.95,0.92}
\lstdefinestyle{mystyle}{
    backgroundcolor=\color{backcolour},   
    commentstyle=\color{codegreen},
    keywordstyle=\color{magenta},
    numberstyle=\tiny\color{codegray},
    stringstyle=\color{codepurple},
    basicstyle=\ttfamily\footnotesize,
    breakatwhitespace=false,         
    breaklines=true,                 
    captionpos=b,                    
    keepspaces=true,                 
    numbers=left,                    
    numbersep=5pt,                  
    showspaces=false,                
    showstringspaces=false,
    showtabs=false,                  
    tabsize=2
}
\allowdisplaybreaks

\renewcommand{\v}[1]{\boldsymbol{\mathbf{#1}}}
\newcommand{\vdot}[1]{\dot{\v{#1}}}
\newcommand{\vhat}[1]{\hat{\v{#1}}}
\newcommand{\vtilde}[1]{\widetilde{\v{#1}}}

\newcommand{\T}{\top}
\newcommand{\vt}[1]{\v{#1}^\T}
\newcommand{\R}[1][]{\mathbb{R}^{#1}}
\newcommand{\bmat}[1]{\begin{bmatrix} #1 \end{bmatrix}}

\newtheorem{theorem}{Theorem}
\newtheorem{lemma}{Lemma}
\newtheorem{remark}{Remark}
\newtheorem{definition}{Definition}
\newtheorem{proposition}{Proposition}

\journal{Elsevier}
\begin{document}
\baselineskip 24pt
\begin{frontmatter}
\author[1]{Zihao Wang}
\author[1]{Yuhan Li}
\author[1]{Yao Shi}
\author[1,2]{Zhe Wu\corref{a}}
\cortext[a]{Corresponding author. E-mail: wuzhe@nus.edu.sg.\\Please refer to \url{https://github.com/killingbear999/ICLRNN} for source codes.}

\address[1]{Department of Chemical and Biomolecular Engineering, National University of Singapore, 117585, Singapore}
\address[2]{Artificial Intelligence Institute, National University of Singapore, 117585, Singapore}
\title{Input Convex Lipschitz Recurrent Neural Networks for Robust and Efficient Process Modeling and Optimization}

\begin{abstract}
Computational efficiency and robustness are essential in process modeling, optimization, and control for real-world engineering applications. While neural network-based approaches have gained significant attention in recent years, conventional neural networks often fail to address these two critical aspects simultaneously or even independently. Inspired by natural physical systems and established literature, input convex architectures are known to enhance computational efficiency in optimization tasks, whereas Lipschitz-constrained architectures improve robustness. However, combining these properties within a single model requires careful review, as inappropriate methods for enforcing one property can undermine the other. To overcome this, we introduce a novel network architecture, termed Input Convex Lipschitz Recurrent Neural Networks (ICL-RNNs). This architecture seamlessly integrates the benefits of convexity and Lipschitz continuity, enabling fast and robust neural network-based modeling and optimization. The ICL-RNN outperforms existing recurrent units in both computational efficiency and robustness. Additionally, it has been successfully applied to practical engineering scenarios, such as chemical process modeling and the modeling and control of Organic Rankine Cycle-based waste heat recovery systems. 
\end{abstract}
	
\end{frontmatter}

\section{Introduction}
\label{sec_introduction}
Neural networks provide a powerful tool to process modeling and control, an important engineering problem in various manufacturing industries such as chemical, pharmaceutical, and energy systems. This work draws inspiration from nature-inspired design to develop a fast and robust solution that addresses two crucial challenges: computational efficiency and robustness, for modeling and control of real-world nonlinear systems. Specifically, neural network models have been developed to capture complex nonlinear dynamics of industrial manufacturing systems, and the models can be further used for process monitoring, predictive maintenance, optimization, and real-time control. However, while neural network-based optimization can optimize process performance such as energy consumption and process profits, based on the prediction of neural networks, computational efficiency remains a big challenge for its real-time implementation. For example, model predictive control (MPC), an advanced optimization-based control technique, has been widely used to manage a process while complying with a set of constraints. Traditional MPC based on first-principles models encounters limitations, particularly in scenarios with intricate dynamics, where deriving these models proves infeasible. Over the past decades, neural network-based MPC (NN-MPC) has received increasing attention, where a neural network that models system dynamics is incorporated into the design of MPC optimization problems \citep{bao2017recurrent, afram2017artificial, lanzetti2019recurrent, ellis2020encoder, nubert2020safe, zheng2022machine, zheng2022online, sitapure2022neural}. However, general optimization problems (e.g., MPC) built on conventional neural networks often fall short of meeting the requirement of computational efficiency. Nonetheless, efforts have been made to mitigate these limitations. For example, the computational efficiency of NN-MPC is enhanced by implementing an input convex architecture in neural networks \citep{wang2025real}. Additionally, robustness of NN-MPC is improved by developing Lipschitz-constrained neural networks (LNNs) \citep{tan2024robust}.

In real-world applications, computational efficiency (for neural network-based optimization tasks) plays a pivotal role due to the imperative need for real-time decision-making, especially in critical engineering processes such as chemical process operations \citep{wang2025real}. We define computational efficiency as the speed at which computational tasks are completed within a given execution timeframe (in actual deployment to industrial engineering processes, the model is trained offline, and the costs associated with the training process are not a primary concern). An effective approach to achieving this efficiency is through the transformation from non-convex to convex structures. Convexity is a ubiquitous characteristic observed in various physical systems such as the potential energy function \citep{goldstein2002classical, taylor2005classical}, magnetic fields \citep{purcell1963electricity, griffiths2005introduction}, and general relativity \citep{hartle2003gravity}. 

Additionally, in real-world engineering applications, robustness plays a critical role due to the prevalent noise present in most real-world data, which significantly hampers neural network performance. Our definition of robustness aligns closely with the principles outlined by \citep{tan2024robust}. Given the inherent noise in real-world process data, our goal is to improve neural networks by effectively learning from noisy training data for comprehensive end-to-end applications. Leveraging LNNs offers a potential solution to the robustness challenge. However, at this stage, the development of neural networks that address the above two challenges simultaneously remains an open question. It is important to note that the integration of input convexity and Lipschitz continuity in a single neural network presents a non-trivial challenge due to their interplay in the way that enforcing one property in the design of neural networks may compromise the other property. 

Motivated by the above considerations, in this work, we introduce an Input Convex Lipschitz Recurrent Neural Network (ICL-RNN) that amalgamates the strengths of input convexity and Lipschitz continuity. This approach aims to achieve the concurrent resolution of computational efficiency and robustness for real-world engineering applications. Specifically, we validate the efficacy of ICL-RNN across various engineering problems that include neural network-based modeling and control for chemical processes and energy systems. Performance assessment hinges on modeling accuracy, computational efficiency (i.e., evaluated by MPC runtime), model complexity (i.e., evaluated by the number of floating point operations (FLOPs) \footnote{
FLOPs, or floating-point operations, refer to the number of floating-point operations (such as addition, subtraction, multiplication, division, etc.) required to call the model once (i.e., a a single forward pass through a model). It serves as a metric for evaluating and quantifying the complexity and computational demand of the model, where higher FLOPs indicate greater complexity.}), and robustness against noise. Our objective is to achieve equilibrium across these dimensions as they hold equal importance in real-world engineering scenarios.

In summary, this work proposes an ICL-RNN architecture, substantiating its theoretical attributes as input convex and Lipschitz continuous, and successfully implements and illustrates that ICL-RNN surpasses state-of-the-art recurrent units in various engineering tasks, including modeling and control of processes in continuous stirred tank reactor (CSTR) systems and Organic Rankine Cycle (ORC)-based waste heat recovery systems. The rest of this article is structured as follows: Sec.~\ref{sec_system} introduces the class of systems and the design of neural network-based MPC. Sec.~\ref{sec_background} provides an overview of the current family of recurrent units. Sec.~\ref{sec_ICLRNN} presents the proposed ICL-RNN architecture and implementation, and proves that the architecture is both input convex and Lipschitz-constrained. Sec.~\ref{sec_ode} uses an example of a chemical process in CSTR systems to validate the performance of ICL-RNN in modeling and control of an ordinary differential equation (ODE) system. Sec.~\ref{sec_orc} applies ICL-RNN to the modeling and control of an ORC-based waste heat recovery system.

\section{Nonlinear Systems and Optimization}
\label{sec_system}
\subsection{Notations}
Let $L_f$ represent the Lipschitz constant of a function $f$. $``\circ"$ denotes composition. $\|\cdot\|$ and $\|\cdot\|_2$ denote the spectral norm and Euclidean norm of a matrix, respectively. $\sigma(\cdot)$ denotes the singular value of a matrix.

\subsubsection{Class of Systems}
We consider a class of nonlinear systems that can be expressed by the ordinary differential equations (ODEs) of the following form:
\begin{equation}\label{eq:classofsystems}
\vdot{x} = F(\v x,\v u)
\end{equation}
where $\v u \in \mathbb{R}^m $ denotes the control action and $\v x \in \mathbb{R}^n $ denotes the state vector. The function $F : \mathbb{X} \times \mathbb{U}  \to \mathbb{R}^n $ is continuously differentiable, where $\mathbb{X} \subset \mathbb{R}^n $ and $\mathbb{U} \subset \mathbb{R}^m $ are connected and compact subsets that enclose an open neighborhood around the origin respectively. In this context, we maintain the assumption throughout this task that $F(0,0) = 0$, ensuring that the origin $(\v x, \v u) = (0, 0)$ serves as an equilibrium point. As it might be impractical to obtain first-principles models for intricate real-world systems, our objective is to construct a neural network to model the nonlinear system specified by Eq.~(\ref{eq:classofsystems}), embed it into optimization problems, and maintain computational efficiency and robustness simultaneously.

\subsubsection{MPC Formulation}
The optimization problem of the MPC that incorporates an ICL-RNN as its predictive model is expressed as follows:
\begin{subequations}
\begin{gather}
\mathcal{J} = \min_{\v u \in S(\Delta)} \int_{t_k}^{t_{k+N}}L(\vtilde{x}(t),\v u(t))dt \label{eq9a}\\
\text{s.t. }\dot{\vtilde{x}}(t) = F_{nn}(\vtilde{x}(t),\v u(t))\label{eq9b}\\
\v x(t) \in \mathbb{X}, \v u(t) \in \mathbb{U}, \ \forall t \in [t_k,t_{k+N})\label{eq9c}\\
\vtilde{x}(t_k) = \v x(t_k)\label{eq9d}
\end{gather}
\label{eq9}
\end{subequations}
\noindent where $S(\Delta)$ denotes the set of piecewise constant functions with period $\Delta$, $\vtilde{x}$ denotes the predicted state trajectory, and $N$ denotes the number of sampling periods in the prediction horizon. The objective function $\mathcal{J}$ outlined in Eq.~(\ref{eq9a}) integrates a cost function $L$  that depends on both the control actions $\v u$ and the system states $\v x$. The system dynamic function $F_{nn}(\vtilde{x}(t), \v u(t))$ expressed in Eq.~(\ref{eq9b}) is parameterized by an RNN (e.g., the proposed ICL-RNN in this work). Eq.~(\ref{eq9c}) encapsulates the constraint function $\mathbb{U}$, delineating feasible control actions.  Eq.~(\ref{eq9d}) establishes the initial condition $\vtilde{x}(t_k)$ in Eq.~(\ref{eq9b}), referring to the state measurement at $t = t_k$.

It is important to highlight that a convex neural network-based MPC remains convex even when making multi-step ahead predictions, where the prediction horizon is greater than one \citep{wang2025real}. This statement holds true if and only if the neural network embedded is inherently input convex and the neural network output is non-negative for certain objective functions (e.g., quadratic functions and absolute functions). In this case, MinMax scaling (i.e., transforming each feature by scaling it to a specified range, such as between 0 and 1, based on the training set) and the ReLU activation function in the final layer can be used to ensure non-negative outputs, thereby preserving the convexity of the optimization problem.

\section{Family of Recurrent Units}
\label{sec_background}
We will first review related works, including recent advancements in high-performing recurrent units such as Input Convex RNNs and Lipschitz RNNs, among others.

\subsubsection{Simple RNN}
A simple RNN cell follows:
\begin{subequations}
\begin{align}
\v h_t & = g_1(\v W^{(x)}\v x_t + \v U^{(h)}\v h_{t-1} + \v b^{(h)}) \\
\v y_t & = g_2(\v W^{(y)}\v h_t + \v b^{(y)})
\end{align}
\label{eq_rnn}
\end{subequations}
where $\v h_t$ is the hidden state at time step $t$, $\v x_t$ is the input at time step $t$, $\v h_{t-1}$ is the hidden state from the previous time step, $\v W^{(x)}$, $\v U^{(h)}$ and $\v W^{(y)}$ are weight matrices for the input, hidden state, and output respectively, $\v b^{(h)}$ and $\v b^{(y)}$ are the bias vectors for the hidden state and output respectively, and $\v y_t$ is the output at time step $t$.

\subsection{High-performing RNNs}
In recent years, several high-performing recurrent units have emerged, including UniCORNN \citep{rusch2021unicornn}, rank-coding based RNNs \citep{jeffares2021spike}, and Linear Recurrent Units \citep{orvieto2023resurrecting}. However, these variants primarily emphasize enhancing accuracy, especially for very long time-dependent sequences modeling, often at the cost of model simplicity and computational efficiency. Moreover, in many engineering applications, accuracy alone does not dictate the suitability of a model. In fact, from an engineering point of view, simplicity and robustness are also important in addition to accuracy, as they ensure the model can be easily adapted and incorporated into existing industrial operational systems, and can be trusted in an uncertain environment with noise and disturbances. Therefore, the goal of this work is to develop a network architecture based on simple RNNs that strikes a balance between these factors while maintaining a satisfactory level of accuracy, rather than investigating the novel designs of RNNs for improved accuracy only.

\subsection{Input Convex RNN}
Building upon nature-inspired design and extending its application to neural networks, Input Convex Neural Networks (ICNNs) is a class of neural network architectures intentionally crafted to maintain convexity in their output with respect to the input. Leveraging these models could be particularly advantageous in various engineering problems, notably in applications such as MPC, due to their inherent benefits especially in optimization. This concept was first introduced as Feedforward Neural Networks (FNN) \citep{amos2017input}, and subsequently extended to Recurrent Neural Networks (RNN) \citep{chen2018optimal}, and Long Short-Term Memory (LSTM) \citep{wang2025real}. 

The initiative behind ICNNs is to leverage the power of neural networks, specifically designed to maintain convexity within their decision boundaries \citep{amos2017input}. ICNNs aim to combine the strengths of neural networks in modeling complex data with the advantages of convex optimization, which ensures the convergence to a global optimal solution. They are particularly attractive for control and optimization problems, where achieving globally optimal solutions is essential. By maintaining convexity, ICNNs help ensure that the optimization problems associated with neural networks (e.g., neural network-based optimization) remain tractable, thus addressing some of the challenges of non-convexity and the potential for suboptimal local solutions in traditional neural networks. A foundational work, known as Input Convex RNN (ICRNN) \citep{chen2018optimal}, serves as one of the baselines. The output of ICRNN follows \cite{chen2018optimal}:
\begin{subequations}
\label{eq_icrnn}
\begin{align}
\v h_t & = g_1(\v U\vhat x_t + \v W\v h_{t-1} + \v D_2\vhat x_{t-1}) \\
\v y_t & = g_2(\v V\v h_t + \v D_1\v h_{t-1} + \v D_3\vhat x_t)
\end{align}
\end{subequations}
The output $\v y_t$ is convex with respect to the input $\vhat x_t$ if non-decreasing and convex activation functions are used for $g_i$, and $\{\v U, \v W, \v V, \v D_1, \v D_2, \v D_3\}$ are chosen to be non-negative weights, where $\vhat x_t$ denotes $\bmat{\vt x_t, -\vt x_t}^\T$.

\subsection{Lipschitz RNN}
By adhering to the definition of Lipschitz continuity for a function $f$, where an $L$-Lipschitz continuous $f$ ensures that any minor perturbation to the input results in an output change of at most $L$ times the magnitude of that perturbation. Therefore, constraining neural networks to maintain Lipschitz continuity significantly fortifies their resilience against input perturbations. This concept was first introduced as FNN \citep{anil2019sorting}, and subsequently extended to RNN \citep{erichson2020lipschitz}, and Convolutional Neural Networks (CNN) \citep{serrurier2021achieving}. A pivotal work, namely the Lipschitz RNN (LRNN) \citep{erichson2020lipschitz}, stands as one of the baselines in this paper. The output of LRNN follows \cite{erichson2020lipschitz}:
\begin{subequations}
\label{eq_lrnn}
\begin{align}
    \vdot{h_t} & = \v A_{\beta_A,\gamma_A} \v h_t + tanh(\v W_{\beta_W,\gamma_W} \v h_t + \v U \v x_t + \v b) \\
    \v y_t & = \v D \v h_t
\end{align}
\end{subequations}
where $\beta_A, \beta_W \in [0,1]$, $\gamma_A, \gamma_W > 0$ are tunable parameters, $\v M_A, \v M_W, \v D, \v U$ are trainable weights, $\vdot{h}$ is the time derivative of $\v h$ (i.e., $\v h$ can be updated by the explicit (forward) Euler scheme or a two-stage explicit Runge-Kutta scheme), and $\v A_{\beta_A,\gamma_A}, \v W_{\beta_W,\gamma_W}$ is computed as follows:
\begin{subequations}
\begin{align}
    \v A_{\beta_A,\gamma_A} & = (1-\beta_A)(\v M_A + \v M_A^T) + \beta_A(\v M_A - \v M_A^T) - \gamma_A \v I \\
    \v W_{\beta_W,\gamma_W} & = (1-\beta_W)(\v M_W + \v M_W^T) + \beta_W(\v M_W - \v M_W^T) - \gamma_W \v I
\end{align}
\end{subequations}

\section{Input Convex Lipschitz Recurrent Neural Networks}
\label{sec_ICLRNN}
In this section, we first introduce a novel Lipschitz-constrained, input convex network within the context of a simple RNN. We then provide theoretical proofs for both the Lipschitz continuity of the network, including an upper bound for its Lipschitz constant, and the input convexity, respectively.

\subsection{Introduction to Input Convex Lipschitz RNN}
We propose a novel Input Convex Lipschitz Recurrent Neural Networks, termed ICL-RNN, that possesses both Lipschitz continuity and input convexity properties. Specifically, the output of ICL-RNN is defined as follows:
\begin{subequations}
\begin{align}
    \v h_t & = g_1(\v W^{(x)} \v x_t + \v U^{(h)} \v h_{t-1} + \v b^{(h)}) \label{eq_hidden_state}\\
    \v y_t & = g_2(\v W^{(y)} \v h_t + \v b^{(y)})
\end{align}
\label{eq_ICLRNN}
\end{subequations}
where $\{\v W^{(x)}, \v W^{(y)}, \v U^{(x)}\}$ are constrained to be non-negative with largest singular values small and bounded by 1, and all $g_i$ are constrained to be convex, non-decreasing, and Lipschitz continuous (e.g., ReLU). Similar to \cite{chen2018optimal}, we {expand the input} as $\vhat x_t = \bmat{\vt x_t, -\vt x_t}^\T\in\R[2n_x]$. Noted that the ICL-RNN maintains the neural architecture of the simple RNN in Eq. \eqref{eq_rnn} while introducing constraints on the weights and activation functions to ensure input convexity and Lipschitz continuity.

A neural network's output inherits input convexity and Lipschitz continuity only if every hidden state possesses these properties. Unlike \cite{chen2018optimal} and \cite{erichson2020lipschitz}, the ICL-RNN enforces these properties by constraining the weights and activation functions rather than adding supplementary variables to the RNN architecture. This approach reduces network complexity and minimizes computational demands. Moreover, models such as LRNN, ICRNN, UniCORNN, and Linear Recurrent Units involve high FLOPs and lack theoretical guarantees on robustness and computational efficiency for neural network-based optimization problems. Many engineering applications can achieve satisfactory modeling accuracy with relatively simple neural network structures. However, improving the robustness and computational efficiency of these models for neural network-based optimization and control remains an important challenge. These considerations motivated the design choice to retain the standard RNN architecture while imposing input convexity and Lipschitz constraints on the weights.

\subsection{Implementation of ICL-RNN}
First, we enforce the non-negative constraint in all weights by applying weight clipping, setting all negative weight values to zero, i.e.,
\begin{equation}
\v W \leftarrow \max(\v W, 0)
\end{equation} 
Next, we perform spectral normalization on the non-negative weights, reducing the spectral norm of individual weights to at most 1, i.e.,
\begin{equation}
\v W \leftarrow \frac{\v W}{\sigma_{max}(\v W) + \epsilon}
\end{equation}
\noindent where $\epsilon$ is a small constant to ensure that operator norm of normalized weight is strictly less than 1, e.g., $1 \times 10^{-3}$. To approximate $\sigma_{max}(\v W)$, we use the power iteration algorithm (see Algorithm~\ref{alg_power_iteration}), an iterative numerical method that computes the largest singular value by repeatedly multiplying the matrix by a vector and normalizing the result until convergence. The complete code is available in \url{https://github.com/killingbear999/ICLRNN}. While spectral normalization has been widely used to stabilize the training of generative adversarial networks (GANs) \citep{miyato2018spectral}, it is less commonly applied to RNNs.

\begin{algorithm}[tb]
\caption{Power Iteration Method}
\label{alg_power_iteration}
\textbf{Input}: $\v W$: the input matrix \\
\textbf{Input}: $maxiter$: maximum number of iterations \\
\textbf{Input}: $eps$: convergence tolerance (stopping criterion) \\
\textbf{Output}: $\sigma_{max}$: largest singular value (approximation) \\
\textbf{Output}: $\v v$: corresponding singular vector (right singular vector) \\
\begin{algorithmic}[1]
\STATE Initialize a random vector $\v v$ of size $n$ (e.g., $\v v \sim Uniform(0, 1)$).
\STATE Normalize $\v v$: $\v v \leftarrow \frac{\v v}{||\v v||}$.
\FOR{$epoch \leftarrow 1,2,\ldots,maxiter$}
\STATE Compute the matrix-vector product: $\v u \leftarrow \v W \v v$.
\STATE Normalize $\v u$: $\v u \leftarrow \frac{\v u}{||\v u||}$.
\STATE Compute the next matrix-vector product: $\v v_{new} \leftarrow \v W^T \v u$.
\STATE Normalize $\v v_{new}$: $\v v_{new} \leftarrow \frac{\v v_{new}}{||\v v_{new}||}$.
\STATE Check for convergence: if $||\v v_{new} - \v v|| < eps$, break.
\STATE Update $\v v \leftarrow \v v_{new}$.
\ENDFOR
\STATE Compute the largest singular value: $\sigma_{max} \leftarrow ||\v W \v v||$.
\STATE Return $\sigma_{max}$ and $\v v$.
\end{algorithmic}
\end{algorithm}

Moreover, \cite{serrurier2021achieving} employs the Björck algorithm to ensure Lipschitz continuity by iteratively refining an initial estimate of an orthogonal matrix. This algorithm minimizes the error between the estimate and the true orthogonal matrix, gradually converging to a stable solution, and adjusts singular values to at most one. However, it is important to emphasize that the Björck algorithm should not be used in the ICL-RNN. Unlike spectral normalization, which scales weights to ensure the spectral norm is bounded and inherently maintains non-negativity, the Björck algorithm introduces both positive and negative values during its orthogonalization process. This behavior is incompatible with the non-negativity constraint required in ICL-RNN.

Additionally, regularization techniques such as dropout \citep{srivastava2014dropout}, batch normalization \citep{ioffe2015batch}, and layer normalization \citep{ba2016layer} can help stabilize neural network training. However, these techniques introduce non-linear transformations, disrupting network convexity and thus should not be used in the training of ICNNs (e.g., the training of our ICL-RNN).

At last, ICL-RNN uses ReLU as the activation function for hidden states instead of the GroupSort function proposed in \cite{anil2019sorting}. This decision aims to maintain convexity within the model architecture. Although GroupSort exhibits gradient norm preservation and higher expressive power \citep{tan2024robust}, its non-convex nature contrasts with our primary objective.

\begin{remark}
For ICNNs, including our proposed ICL-RNN, we recommend applying MinMax scaling to the data prior to fitting it to the neural network. MinMax scaling is a normalization technique that transforms data to a specified range, usually between 0 and 1 (or another specified range). This technique {linearly} scales each feature $x$ in the dataset $\mathbb{X}$ individually based on its minimum and maximum values, i.e., $\forall x \in \mathbb{X}$:
\begin{equation}
  x \leftarrow \frac{x - \min(\mathbb{X})}{\max(\mathbb{X}) - \min(\mathbb{X})}(b - a) + a
\end{equation}
where $\min(\mathbb{X})$ is the minimum value of the feature in the dataset, $\max(\mathbb{X})$ is the maximum value of the feature in the dataset, $a$ is the lower bound of the desired range (default is 0), and $b$ is the upper bound of the desired range (default is 1). This preprocessing step ensures that all feature values lie within the specified range, which can improve the stability and performance of the neural network.
\end{remark}

\subsection{Lipschitz Continuity of ICL-RNN}
We first present the definition of Lipschitz continuity and related lemmas for readers' reference, as follows:
\begin{definition}[\citep{eriksson2013applied}]
\label{def_lipschitz}
A function $f: \mathbb{X} \rightarrow \mathbb{R}$ is Lipschitz continuous with Lipschitz constant $L_f$ (i.e., $L_f$-Lipschitz) if and only if the following inequality holds $\forall~x_1, x_2 \in \mathbb{X}$:
\begin{equation}
\label{eq_lipschitz}
    \|f(x_1) - f(x_2)\|_2 \leq L_f\|x_1 - x_2\|_2
\end{equation}
\end{definition}

\begin{definition}
\label{def_spectral_norm}
Given a linear function $f(\v x) = \v W \v x$ with a weight $\v W \in \mathbb{R}^{n \times m}$, we have:
\begin{equation}
\label{eq_spectral_norm}
    L_f = \|\v W\| = \sigma_{max}(\v W)
\end{equation}
\end{definition}

\begin{lemma}
\label{lemma_sum}
Consider a recurrence relation of the form, omitting the bias term:
\begin{equation}
    \v h_t = g(\v W^{(x)}\v x_t + \v U^{(h)}\v h_{t-1})
\end{equation}
where $g$ represents a general transformation applied to the weighted inputs, $\v W^{(x)}$ and $\v U^{(h)}$ are weight matrices, and $\v x_t$ and $\v h_{t-1}$ are the inputs. The terms $\v W^{(x)}\v x_t$ and $\v U^{(h)}\v h_{t-1}$ contribute {independently} to the change in $\v h_t$. The Lipschitz constant $L_{h_t}$ of the mapping $\v x_t, \v h_{t-1} \rightarrow \v h_t$ satisfies:
\begin{equation}
    L_{h_t} \leq \max(\|\v W^{(x)}\|, \|\v U^{(h)}\|) \times L_g
\end{equation}
where $\|\v W^{(x)}\|$ equals the Lipschitz constant (or spectral norm) of the weight matrix $\v W^{(x)}$ and $\|\v U^{(h)}\|$ equals the Lipschitz constant (or spectral norm) of the weight matrix $\v U^{(h)}$.
\end{lemma}

\begin{lemma}[\citep{eriksson2013applied}]
\label{lemma_product_composition}
Consider a set of functions, e.g., $f : \mathbb{X} \rightarrow \mathbb{R}$ and $g : \mathbb{X} \rightarrow \mathbb{R}$ with Lipschitz constants $L_f$ and $L_g$ respectively, by taking their composition, denoted as $h = f \circ g$, the Lipschitz constant $L_h$ of the resultant function $h$ satisfies:
\begin{equation}
\label{eq_product_composition}
    L_h \leq L_f \times L_g
\end{equation}
\end{lemma}

\begin{lemma}[\citep{virmaux2018lipschitz}]
\label{lemma_activation}
The most common activation functions such as ReLU, Sigmoid, and Tanh have a Lipschitz constant that equals 1, while Softmax has a Lipschitz constant bounded by 1.
\end{lemma}

Based on the previously stated lemmas, we present the following proposition regarding the upper bound for the Lipschitz constant of the ICL-RNN.

\begin{proposition}
\label{prop_lip}
Consider the $L$-layer ICL-RNN. The Lipschitz constant of the output $\v y_t$ is upper bounded by 1 if and only if all weights $\{\v W^{(x)}, \v U^{(h)}, \v W^{(y)}\}$ and all activation functions $g_i$ have a Lipschitz constant upper bounded by 1.
\end{proposition}

\begin{proof}
By unrolling the recurrent operations over time $t$ in Eq.~(\ref{eq_ICLRNN}), the output of ICL-RNN can be described as a function, omitting the bias term:
\begin{equation}
\label{eq_RNN_unrolled}
    \v y_t = g_t^{(y)}(\v W^{(y)}g_{t}^{(h)}(\v W^{(x)} \v x_{t} + \v U^{(h)}g_{t-1}^{(h)}(\v W^{(x)} \v x_{t-1} +  \v U^{(h)} \cdots g_1^{(h)}(\v W^{(x)} \v x_1 + \v U^{(h)} \v h_0))))
\end{equation}
Thus, the Lipschitz constant of an ICL-RNN is upper bounded by the product of the individual Lipschitz constants (i.e., Lemma~\ref{lemma_sum} and Lemma~\ref{lemma_product_composition}):
\begin{equation}
\label{eq_LRNN_bound}
    L_{y_t} \leq L_{g_t^{(y)}} \times \|\v W^{(y)}\| \times L_{g_{t}^{(h)}} \times \max(\|\v W^{(x)}\|, \|\v U^{(h)}\| \times \cdots \times L_{g_1^{(h)}} \times \max(\|\v W^{(x)}\|, \|\v U^{(h)}\|))
\end{equation}
By ensuring that the Lipschitz constants of all weights $\{\v W^{(x)}, \v U^{(h)}, \v W^{(y)}\}$ in the ICL-RNN are upper bounded by 1 through spectral normalization, and that all activation functions $g_i$, such as ReLU, have Lipschitz constants also bounded by 1, the Lipschitz constant of the final output of the ICL-RNN is guaranteed to be upper bounded by 1.
\end{proof}

\begin{remark}
\label{remark_numerical}
Since the function approximated by RNNs is not analytically tractable, we estimate the Lipschitz constant $L_{nn}$ of the neural network numerically, as follows:
\begin{equation}
    L_{nn} \approx \max_{x,y \in \mathbb{D}}\frac{||f(x) - f(y)||}{||x-y||}
\end{equation}
where $\mathbb{D}$ is the domain of interest.  To approximate $L$, we generate random pairs of $x, y \in \mathbb{D}$ and compute the ratio for each pair, increasing the sampling density to enhance the accuracy of the estimate. For more accurate techniques to estimate the Lipschitz constant of deep neural networks, interested readers are encouraged to refer to \cite{virmaux2018lipschitz, fazlyab2019efficient, latorre2020lipschitz}.
\end{remark}

\subsection{Input Convexity of ICL-RNN}
Similarly, we present the definition of convexity for readers' reference, as follows:
\begin{definition}[\citep{boyd2004convex}]
\label{def_convexity}
A function $f: \mathbb{X} \rightarrow \mathbb{R}$ is convex if and only if the following inequality holds $\forall~(x_1, x_2) \in \mathbb{X}$, with $x_1 \neq x_2$, and  $\forall~\lambda \in (0, 1)$:
\begin{equation}
\label{eq_convexity}
    f(\lambda x_1 + (1-\lambda)x_2) \leq \lambda f(x_1) + (1-\lambda)f(x_2)
\end{equation}
and it is strictly convex if the relation $\leq$ becomes $<$.
\end{definition}

Next, we develop the following proposition to show the convexity of $L$-layer ICL-RNN.
\begin{proposition}
\label{prop_convex}
Consider the $L$-layer ICL-RNN. The output $\v y_t$ is input convex if and only if all weights $\{\v W^{(x)}, \v U^{(h)}, \v W^{(y)}\}$ are non-negative and all activation functions $g_i$ are convex and non-decreasing.
\end{proposition}

\begin{proof}
With Eq. \eqref{eq_RNN_unrolled}, the proof directly follows from the fact that affine transformations with non-negative matrices and compositions of convex non-decreasing functions preserve convexity \citep{boyd2004convex}. 
\end{proof}

With Propositions~\ref{prop_lip} and~\ref{prop_convex}, we present the following theorem to demonstrate that the ICL-RNN is both input convex and Lipschitz-constrained.

\begin{theorem}
Consider the $L$-layer ICL-RNN. Each element of the output $\v y_t$ is a convex, non-decreasing function of the input $\vhat x_t=\bmat{\vt x_t,-\vt x_t}^\T$ (or just $\v x_t$), with the Lipschitz constant upper bounded by 1, at each time step $t=1,2,\dots,t$, for all $\vhat x_t \in D\times D$ in a convex feasible space, provided that the following conditions are met: (1) All weights are non-negative, and their largest singular value is at most 1 at each time step $t$; (2) All activation functions are convex, non-decreasing, and Lipschitz continuous (e.g., ReLU) at each time step $t$.
\end{theorem}

\begin{proof}
    Using Proposition~\ref{prop_lip} and Proposition~\ref{prop_convex}, we can demonstrate that the Lipschitz constant of the ICL-RNN is upper-bounded by 1 and that it is input-convex. The Lipschitz constant is bounded by 1 because the weight matrices are spectral-normalized. Furthermore, since the weights are non-negative and the activation functions are convex and non-decreasing, the overall function remains convex. The proof follows directly from the fact that spectral normalization preserves convexity, as it is a linear transformation. Therefore, the ICL-RNN is both input-convex and Lipschitz-constrained.
\end{proof}

\section{Case Study 1: Application To a Chemical Process Example}
\label{sec_ode}
\subsection{Process Description}
In this section, we demonstrate the effectiveness of ICL-RNN in a chemical reactor example of a CSTR system, which incorporates a heating jacket responsible for supplying or extracting heat at a rate $Q$ and facilitates the conversion of reactant $A$ into product $B$.  The dynamic model of the CSTR is characterized by the ensuing material and energy balance equations:
\begin{subequations}
\begin{align}
    \frac{dC_A}{dt} & = \frac{F}{V_L}(C_{A0}-C_A)-k_0e^{\frac{-E}{RT}}C_A^2 \\
    \frac{dT}{dt} & = \frac{F}{V_L}(T_0-T)+\frac{-\triangle H}{\rho_LC_p}k_0e^{\frac{-E}{RT}}C_A^2+\frac{Q}{\rho_LC_pV_L}
\end{align} 
\label{eqcstr}
\end{subequations}
\noindent where $C_A$ denotes the concentration of reactant $A$, $F$ denotes the volumetric flow rate, $V_L$ denotes the volume of the reacting liquid in the tank, $C_{A0}$ denotes the inlet concentration of reactant $A$, $k_0$ denotes the pre-exponential constant, $E$ denotes the activation energy, $R$ denotes the ideal gas constant, $T$ denotes the temperature, $T_0$ denotes the inlet temperature, $\Delta H$ denotes the enthalpy of reaction, $\rho_L$ denotes the constant density of the reactant, $C_p$ denotes the heat capacity, and $Q$ denotes the heat input rate. Moreover, $\v x^T = [C_A - C_{As}, T - T_s]$ denotes the system states, where $T - T_s$ and $C_A - C_{As}$ represent the deviations in the temperature of the reactor and the concentration of $A$ from their respective steady-state values. $\v u^T = [\Delta C_{A0}, \Delta Q]$ denotes the control actions, where $\Delta Q = Q - Q_s$ and $\Delta C_{A0} = C_{A0} - C_{A0_s}$ denote the manipulated inputs within this system, representing the alterations in the heat input rate and the inlet concentration of reactant $A$, respectively. In summary, the inputs consist of $[T_t - T_s, C_{A,t} - C_{As}, \Delta Q_t, \Delta C_{A0,t}]$ at the current time step $t$. The outputs entail the state trajectory $[T_{t+1:n} - T_s, C_{A,t+1:n} - C_{As}]$ over the subsequent $n$ time steps. In this case, $n$ is set to $10$. We conduct open-loop simulations for the CSTR of Eq.~(\ref{eqcstr}) using explicit Euler method. These simulations aim to collect data that mimics the real-world data for neural network training purposes. The detailed parameter values are listed in Table~\ref{tab_cstr_values}.

\begin{table}[htbp]
\centering
\caption{Parameter values for the CSTR model}
\resizebox{0.8\linewidth}{!} {
\renewcommand{\arraystretch}{1.5} 
\begin{tabular}{>{\centering\arraybackslash}p{2.5cm}|>{\centering\arraybackslash}p{4.5cm}|>{\centering\arraybackslash}p{2.5cm}|>{\centering\arraybackslash}p{4.5cm}}
\hline
\textbf{Parameter} & \textbf{Value} &\textbf{Parameter} & \textbf{Value} \\ \hline
$T_s$ & $402~\mathrm{K}$ & $C_{A_s}$ & $1.95~\mathrm{kmol/m^3}$ \\
$F$ & $5~\mathrm{m^3/hr}$ & $V_L$ & $1~\mathrm{m^3}$ \\
$Q_s$ & $0.0~\mathrm{kJ/hr}$ & $T_0$ & $300~\mathrm{K}$ \\
$R$ & $8.314~\mathrm{kJ/(kmol \cdot K)}$ & $C_{A0_s}$ & $4~\mathrm{kmol/m^3}$ \\
$\rho_L$ & $1000~\mathrm{kg/m^3}$ & $C_p$ & $0.231~\mathrm{kJ/(kg \cdot K)}$ \\
$E$ & $5 \times 10^4~\mathrm{kJ/kmol}$ & $k_0$ & $8.46 \times 10^6~\mathrm{m^3/(kmol \cdot hr)}$ \\
$\Delta H$ & $-1.15 \times 10^4~\mathrm{kJ/kmol}$ & & \\
\hline
\end{tabular}
}
\label{tab_cstr_values}
\end{table}

\subsection{Modeling Performance}
\label{sec_cstr_modeling_performance}
To evaluate the model's robustness against noise, we introduce Gaussian noise to the training data (i.e., output trajectory $\v y$), in an additive manner \footnote{Additive noise is more common in general applications like machine learning, sensor data, and most engineering fields because it aligns with the assumption of independent, Gaussian noise, while relative noise is more specific to cases where the noise source is tied to the signal's magnitude, such as in financial systems, biology, and some physics domains.} as follows:
\begin{equation}
\label{eq_Gaussian}
    \v y_i \leftarrow \v y_i + z_i
\end{equation}
where $z_i$ is independently and identically distributed, drawn from a zero-mean normal distribution with variance $N$ (the degree of noise), i.e, $z \sim \mathcal{N}(0, N)$. All models undergo training and evaluation using uniform configurations, incorporating a MinMax Scaling (i.e., scaling the data into a specific range between 0 and 1), a batch size of 256, a two-layer neural architecture with 64 neurons per layer (i.e., model size of (64, 64)), the mean-squared error (MSE) loss function, and the Adam optimizer. Detailed hyperparameter values, including the FLOPs to quantify model complexity, are listed in Table~\ref{tab_hyper}.

\begin{table}[htbp]
\centering
\caption{Hyperparameters of neural network models}
\vspace{.5em}
\resizebox{0.8\linewidth}{!} {
\renewcommand{\arraystretch}{1.5} 
\begin{tabular}{>{\centering\arraybackslash}p{3.5cm}|>{\centering\arraybackslash}p{3.5cm}|>{\centering\arraybackslash}p{3.5cm}|>{\centering\arraybackslash}p{3.5cm}}
\hline
\textbf{Model} & \textbf{Activation} & \textbf{Inputs} & \textbf{FLOPs}\\
\hline
Plain RNN & Tanh & $\vt x_t$& 27,924\\
Plain LSTM & Tanh & $\vt x_t$ & 104,468\\
LRNN & Tanh & $\vt x_t$ & 159,508\\
ICRNN & ReLU & $\vhat x_t=\bmat{\vt x_t,-\vt x_t}^\T$ & 79,892\\
ICL-RNN (Ours) & ReLU & $\vhat x_t=\bmat{\vt x_t,-\vt x_t}^\T$ & 28,436\\
\hline
\end{tabular}}
\label{tab_hyper}
\end{table}

In Fig.~\ref{fig_cstr_loss}, the left plot shows the mean testing MSE versus the degree of additive noise over three random trials in a linear scale, while the right plot presents the mean and standard deviation of the testing MSE versus the degree of additive noise over three random trials in a logarithmic scale for better clarity. Similarly, in Fig.~\ref{fig_cstr_lip}, the left plot illustrates the mean Lipschitz constant of the trained model versus the degree of additive noise over three random trials in a linear scale, and the right plot displays the mean and standard deviation of the Lipschitz constant versus the degree of additive noise in a logarithmic scale for improved visualization. The Lipschitz constant of the trained model is estimated using numerical methods (see Remark~\ref{remark_numerical}). The solid line represents the mean across three trials on different unseen reactions, while the shaded region indicates the standard deviation.

\begin{figure}[ht!]
    \centering
    \begin{subfigure}[t]{0.48\textwidth}
        \centering
        \includegraphics[width=\columnwidth]{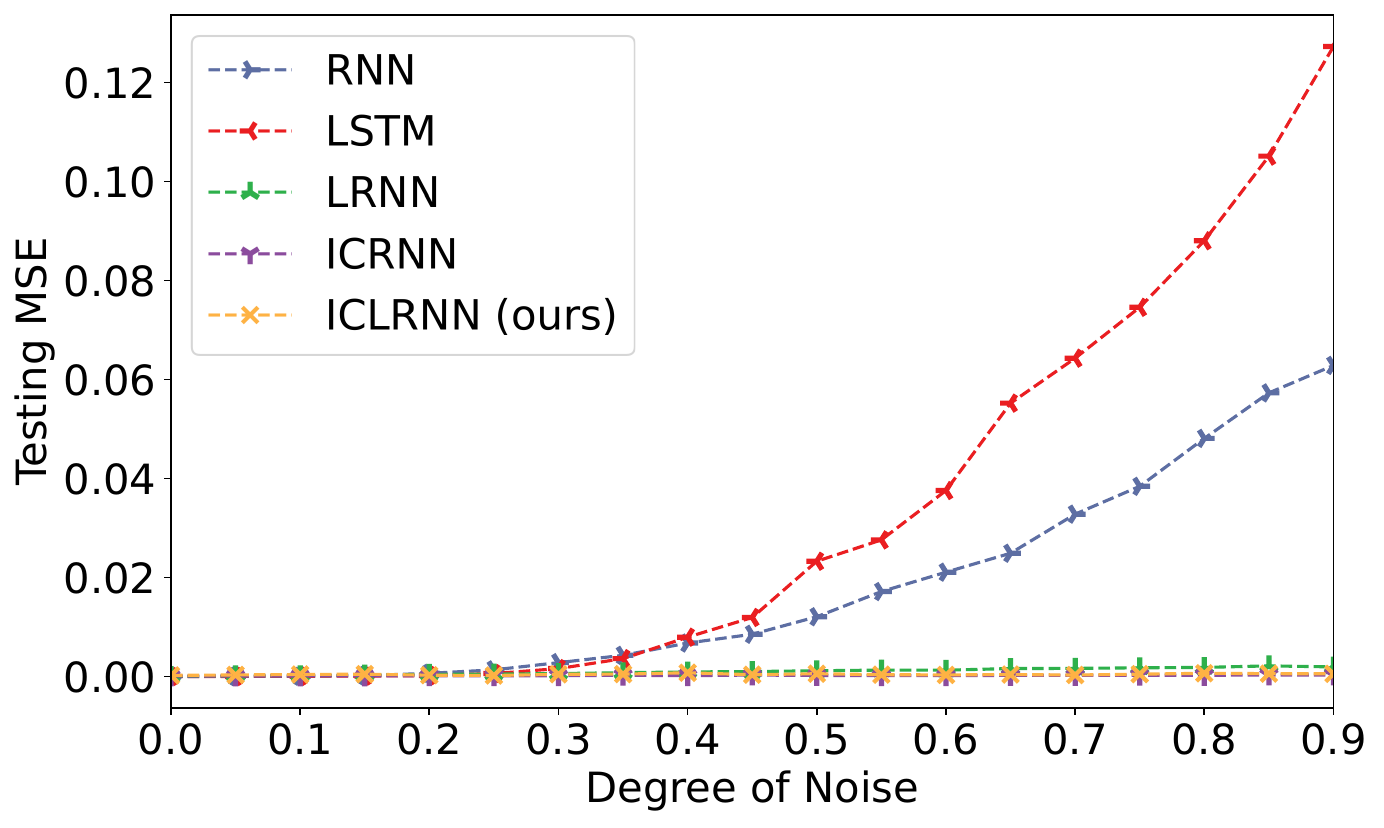}
        \caption{Linear scale}
        \label{fig_linear_loss}
    \end{subfigure}
    ~ 
    \begin{subfigure}[t]{0.48\textwidth}
        \centering
        \includegraphics[width=\columnwidth]{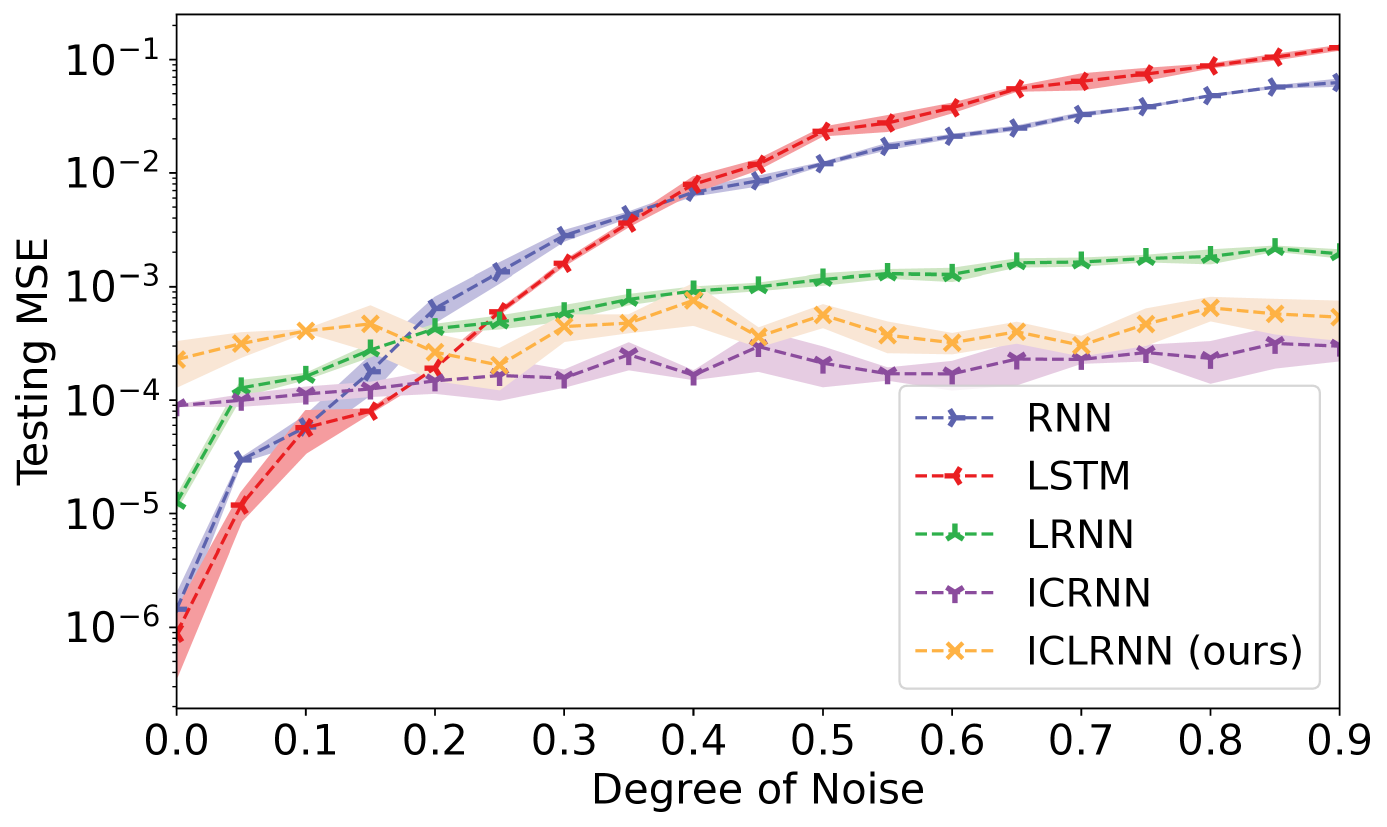}
        \caption{Logarithmic scale}
        \label{fig_log_loss}
    \end{subfigure}
    \caption{Mean testing MSE vs. degree of additive noise over 3 random trials.}
    \label{fig_cstr_loss}
\end{figure}

\begin{figure}[ht!]
    \centering
    \begin{subfigure}[t]{0.48\textwidth}
        \centering
        \includegraphics[width=\columnwidth]{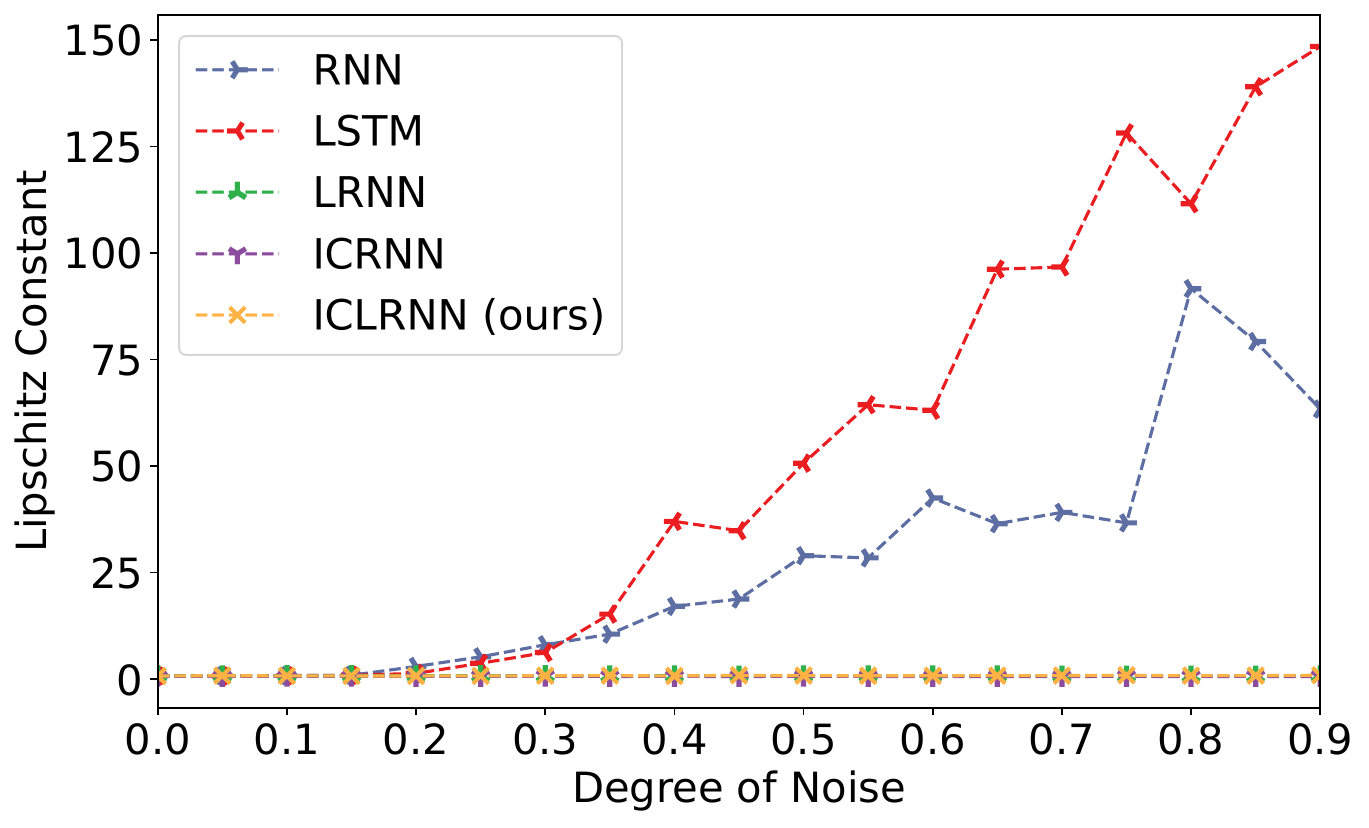}
        \caption{Linear scale}
        \label{fig_linear_lip}
    \end{subfigure}
    ~ 
    \begin{subfigure}[t]{0.48\textwidth}
        \centering
        \includegraphics[width=\columnwidth]{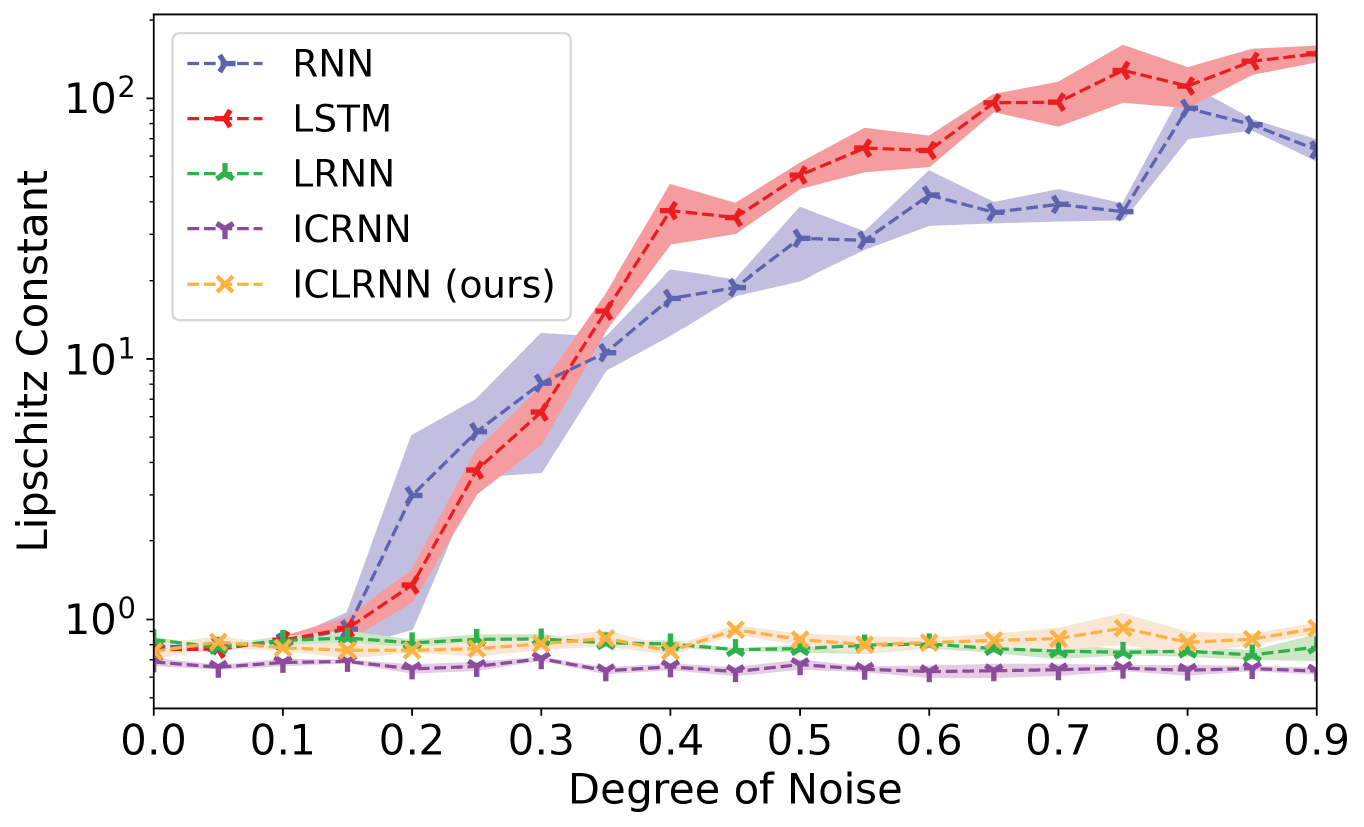}
        \caption{Logarithmic scale}
        \label{fig_log_lip}
    \end{subfigure}
    \caption{Mean Lipschitz constant of the trained model vs. degree of additive noise over 3 random trials.}
    \label{fig_cstr_lip}
\end{figure}

From Figs.~\ref{fig_cstr_loss} and~\ref{fig_cstr_lip}, it is shown that conventional recurrent units, such as standard RNN and LSTM, lack robustness against noise. The higher Lipschitz constant observed for LSTM compared to RNN under increasing noise levels may result from overfitting to noisy data due to its advanced gating mechanisms. In contrast, LRNN, ICRNN, and ICL-RNN demonstrate robustness across different noise levels, maintaining small Lipschitz constants (less than 1), which implies that they exhibit 1-Lipschitz behavior across various noise levels. Although ICRNN does not theoretically satisfy the Lipschitz property, the convexity constraint effectively enhances its robustness to noise. This robustness is because the convexity constraint simplifies the model's learning capability, enabling it to learn non-convex functions (such as noisy data) through convex approximations. This simplification inherently helps prevent ``overfitting'' to noise during training.

The advantage of ICL-RNN over LRNN lies in its lower testing MSE, indicating higher accuracy in modeling the system (see Fig.~\ref{fig_cstr_loss}), and its significantly lower FLOPs, reflecting reduced model complexity. Specifically, for a 2-layer architecture with 64 neurons per layer, LRNN requires 159,508 FLOPs, whereas ICL-RNN only requires 28,436 FLOPs - approximately 5.6 times fewer (see Table~\ref{tab_hyper}). This demonstrates that a less complex model can achieve more accurate results.

Similarly, ICL-RNN offers a significant advantage over ICRNN in terms of FLOPs, representing a less complex model under the same neural architecture. For a 2-layer architecture with 64 neurons per layer, ICRNN requires 79,892 FLOPs, while ICL-RNN only requires 28,436 FLOPs - approximately 2.8 times fewer (see Table~\ref{tab_hyper}). Additionally, ICRNN has a theoretical structural limitation in that its hypothesis space cannot be very large. Due to the introduction of additional trainable parameters in the standard RNN architecture, ICRNN performs well when the hypothesis space is small (e.g., 2 layers with fewer than 128 neurons per layer). However, as the hypothesis space increases (e.g., 2 layers with more than 128 neurons per layer), training becomes unstable (see Table~\ref{tab_nan}). This instability results in the MSE becoming Not a Number (NaN) \footnote{In computing, NaN, which stands for ``Not a Number'', is a numeric data type that represents an undefined or unrepresentable value. It can be used to represent: infinity, the result of dividing by zero, a missing value, and the square root of a negative number.}, rendering the model untrainable unless stabilization techniques such as layer normalization or batch normalization are employed. However, such techniques should be avoided in ICNN training because they introduce nonlinear scaling, violating the input convexity property of the final model. In summary, as the hypothesis space grows (e.g., by increasing the number of neurons), ICRNN's testing MSE rises and eventually becomes NaN, whereas ICL-RNN maintains a small and stable testing MSE.

\begin{table}[htbp]
\centering
\vspace{.5em}
\caption{Testing MSE of ICRNN and ICL-RNN with respect to increasing hypothesis space}
\label{tab_nan}
\resizebox{1\linewidth}{!} {
\renewcommand{\arraystretch}{1.5} 
\begin{tabular}{c|c|c|c|c}
\hline
& \multicolumn{2}{c|}{\textbf{Testing MSE}} & \multicolumn{2}{c}{\textbf{Lipschitz Constant}}\\
\cline{2-5}
\textbf{Model Size} & \textbf{ICRNN} & \textbf{ICL-RNN} & \textbf{ICRNN} & \textbf{ICL-RNN} \\
\hline
(32, 32) & $\bm{8.87 \times 10^{-5} \pm 3.964 \times 10^{-7}}$ & $4.61 \times 10^{-4} \pm 1.899 \times 10^{-4}$ & $6.21 \times 10^{-1} \pm 5.547 \times 10^{-3}$ & $7.35 \times 10^{-1} \pm 1.148 \times 10^{-1}$ \\
(64, 64) & $\bm{9.01 \times 10^{-5} \pm 2.423 \times 10^{-7}}$ & $1.63 \times 10^{-4} \pm 3.879 \times 10^{-5}$ & $6.88 \times 10^{-1} \pm 2.154 \times 10^{-2}$ & $7.62 \times 10^{-1} \pm 6.123 \times 10^{-2}$ \\
(128, 128) & $3.92 \times 10^{-4} \pm 6.278 \times 10^{-5}$ & $\bm{2.89 \times 10^{-4} \pm 8.914 \times 10^{-5}}$ & $8.96 \times 10^{-1} \pm 2.392 \times 10^{-2}$ & $8.24 \times 10^{-1} \pm 4.015 \times 10^{-2}$ \\
(256, 256) & $3.58 \times 10^{2} \pm 9.456$ & $\bm{7.69 \times 10^{-5} \pm 1.891 \times 10^{-6}}$ & $2.92 \times 10^{1} \pm 2.738$ & $7.05 \times 10^{-1} \pm 7.269 \times 10^{-3}$ \\
(512, 512) & NaN & $\bm{2.25 \times 10^{-4} \pm 2.007 \times 10^{-5}}$ & NaN & $7.39 \times 10^{-1} \pm 4.494 \times 10^{-2}$\\
\hline
\end{tabular}
}
\end{table}
\subsection{Computational Efficiency in MPC for CSTR Control}

In addition to robustness against noise, model computational efficiency is also a significant factor in process modeling, optimization, and control applications. In this work, we employ the framework of MPC to evaluate the models’ computational efficiency, which is an advanced control methodology widely implemented in industrial chemical plants and energy systems. Specifically, a typical MPC consists of model prediction in the horizons of a fixed temporal window, optimal solving of a properly set objective function, and actuator correction by updating the first horizon’s optimal solution. Therefore, the computational efficiency of MPC largely depends on how rapidly the optimal solution can be obtained, which presents challenges to the model’s tractability in optimization problems, especially those of systems with strong nonlinearity.

We first evaluate the computational efficiency of MPC using different neural network models for the CSTR system. In particular, the initial states $ C_A - C_{As}$ and $ T - T_s$ are set to $-1.65\ \mathrm{kmol/m^3}$ and $72\ \mathrm{K}$, with the set points of both set to 0.  The closed-loop state profiles under the MPC using RNN, LSTM, LRNN, ICRNN, and ICLRNN are shown in Fig.~\ref{fig_cstr_results}, and during the control processes, it is observed that the computation time of ICLRNN-MPC and ICRNN-MPC is notably lower than that of the others. Specifically, the average computation time of solving MPC for each time step is reported in Table~\ref {tab_mpctime}.


\begin{figure}[ht!]
    \centering
    \begin{subfigure}[t]{0.49\textwidth}
        \centering
        \includegraphics[width=\textwidth]{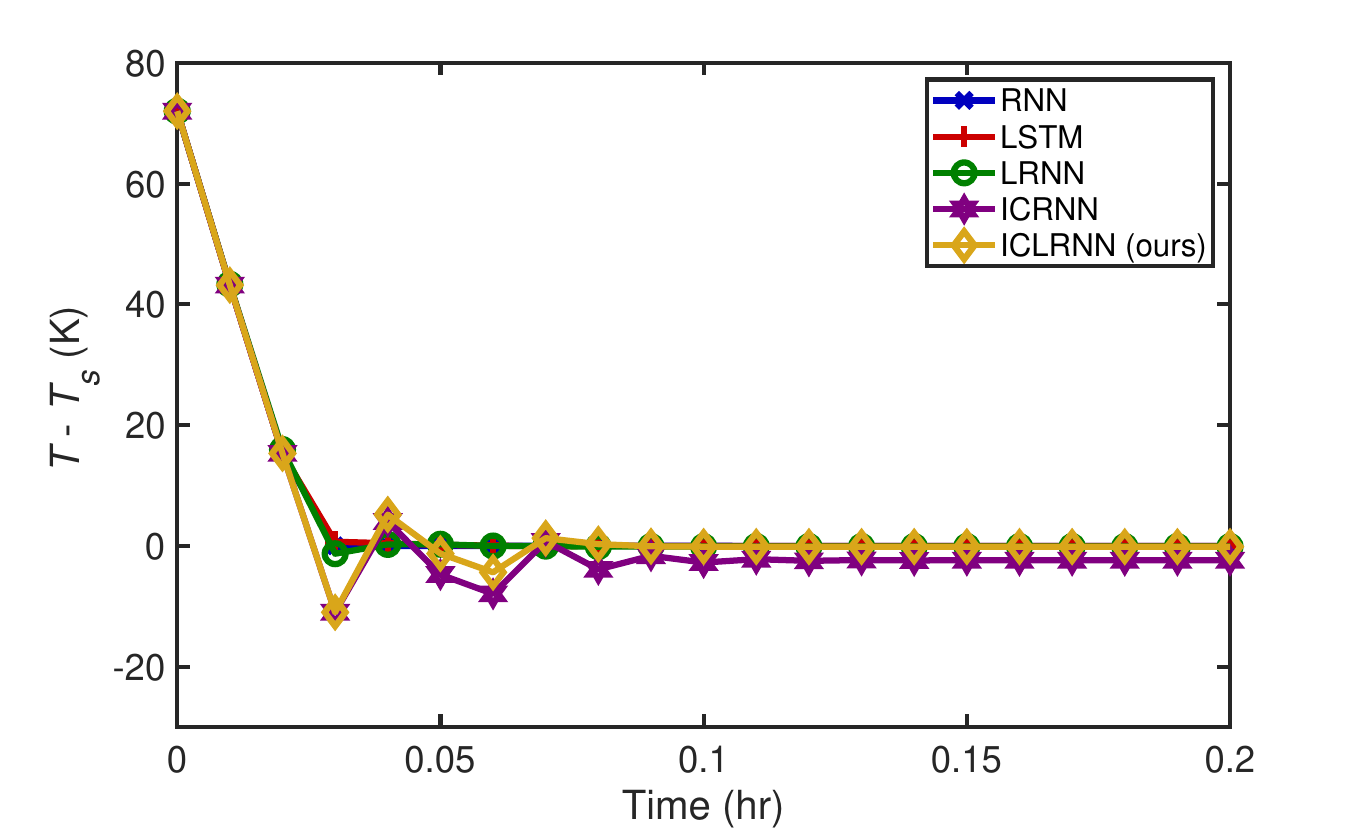}
        \label{fig_cstr_T}
    \end{subfigure}
    \hspace{0.0001\textwidth}
    \begin{subfigure}[t]{0.49\textwidth}
        \centering
        \includegraphics[width=\textwidth]{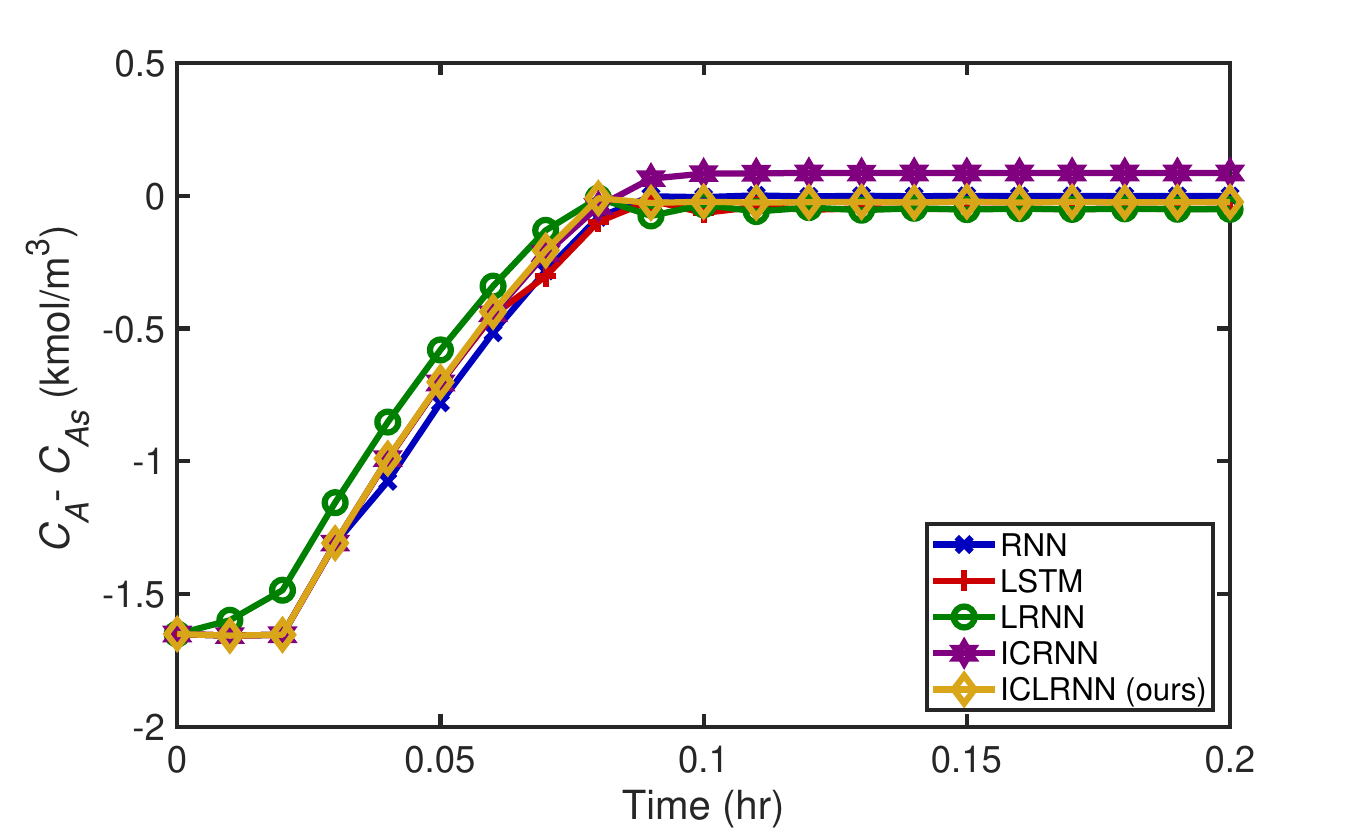}
        \label{fig_cstr_Ca}
    \end{subfigure}

    \caption{Closed-loop state profiles under neural network model-based MPC.}
    \label{fig_cstr_results}
\end{figure}

\begin{table}[htbp] 
\centering
\caption{MPC computation time of neural network models}
\vspace{.5em}
\resizebox{0.8\linewidth}{!} {
\renewcommand{\arraystretch}{1.5} 
\begin{tabular}{>{\centering\arraybackslash}p{4cm}|>{\centering\arraybackslash}p{10cm}}
\hline
\textbf{Model} & \textbf{Average computation time (s)}\\
\hline
Plain RNN & 152.254
\\
Plain LSTM & 149.494\\
LRNN & 141.467\\
ICRNN & 81.779\\
ICL-RNN (Ours) & 100.997\\
\hline
\end{tabular}}
\label{tab_mpctime}
\end{table}
The average computation time of ICLRNN-MPC is 33.67\%, 32.44\%, and 28.61\% lower than that of RNN-MPC, LSTM-MPC, and LRNN-MPC, respectively, and is 19.03\%  higher than that of ICRNN-MPC. This advantage in the computational efficiency of ICLRNN and ICRNN is due to the convex optimization introduced in their training process. 

Furthermore, while the introduction of the Lipschitz constraint makes the computational efficiency of ICLRNN-MPC slightly lower than that of ICRNN-MPC, it effectively enhances the robustness against noise (as shown in Sec.~\ref{sec_cstr_modeling_performance})  without significantly compromising the computational efficiency. In this case, the proposed ICLRNN integrates strong robustness with high computational efficiency by integrating convex optimization and Lipschitz constraint. This synergy provides practical value in real-world process control, where models’ tolerance for data noise and computational efficiency during operation are both highly required. 

\section{Modeling and Control of ORC-Based Waste Heat Recovery System}
\label{sec_orc}
\subsection{ORC-Based Waste Heat Recovery System}

In this section, we consider an ORC-based waste heat recovery system, a more complex example in energy systems, to show the performance of the proposed ICLRNN.  
Specifically, the ORC system is designed to recover the energy of waste heat and convert it into mechanical power. A typical ORC system mainly comprises an evaporator, pump, reservoir, condenser, and expander, through which the organic working fluid R245fa circulates to carry out the energy conversion process. The dynamics of the system can be summarized as follows: The pump transfers the liquid  R245fa to the evaporator, where it absorbs heat from the waste heat source and undergoes phase transition into vapor, which then expands through the expander to generate the shaft work. Subsequently, the vapor is condensed back to liquid form in the air-cooled condenser and recirculated to the evaporator, thereby completing the ORC  circulation. 

The first-principles model of the ORC system adopted in this study is based on \citep{shi2023orc},  
The reservoir outlet temperature and pressure are considered to be constants due to the assumption of a sufficiently large reservoir. Furthermore, given that the condenser is typically subject to negligible disturbances in practical operation, the overall system dynamics is predominantly governed by the heat and mass transfer processes within the pump, expander, and most critically, the evaporator. In this case, the models of the expander and pump are  given as follows:
\begin{subequations}
\begin{align}
    \dot{m}_{exp} & = \eta_{v, exp}\rho_{o, ex}V_{exp}\omega_{exp}\\
    h_{o, exp} & = h_{i, exp} + \eta_{ise, exp}(h_{ise, exp}-h_{i, exp})
\end{align} 
\end{subequations}
\begin{subequations}
\begin{align}
    \dot{m}_{pump} & = \eta_{v, pump}\rho_{o, pump}V_{pump}\omega_{pump}\\
    h_{o, pump} & = h_{i, pump} + \frac{1}{\eta_{ise, pump}}(h_{ise, pump}-h_{i, pump})
\end{align} 
\end{subequations}
where $\dot{m}$ denotes the mass flow rates through the expander and pump, with subscripts ``exp'' and ``pump'' representing the expander and pump, respectively. $\eta_{v}$, $\rho_{o}$, $V$, and $\omega$ are the corresponding volumetric efficiency, outlet fluid density, displacement volume, and rotational speed, respectively; $h_{i}$ and $h_{o}$ represent the inlet and outlet specific enthalpy, while $h_{ise}$ and $\eta_{ise}$ denote the isentropic outlet enthalpy and efficiency.
Regarding crucial evaporator dynamics, the moving boundary (MB) modeling method used in \citep{shi2023orc} divides the evaporator into three regions (subcooled, two-phase, superheated), corresponding to the phase states of the working fluid along the evaporator. Consequently, the mass and energy conservation equations for each region ($k=1, 2, 3$ represent the subcooled, two-phase, superheated regions of the evaporator, respectively) are formulated as follows: 
\begin{equation}
\label{eq_Aeva}
    \frac{\partial A_{eva}\rho^k_{eva}}{\partial t} = \frac{\partial \dot{m}^k_{eva}}{\partial z}
\end{equation}
\begin{equation}
\label{eq_hkeva}
\frac{\partial \dot{m}^k_{eva} h^k_{eva}}{\partial z} = - \frac{\partial(A_{eva}\rho^k_{eva} h^k_{eva}-A_{eva}\rho^k_{eva})}{\partial t} + \pi\alpha^k_{i, eva} D^k_{i, eva}(T^k_{w, eva}-\bar{T}^k_{r, eva})
\end{equation}
where $A_{eva}$ is the flow cross-sectional area of the evaporator channel. $\rho^k_{eva}$, $\dot{m}^k_{eva}$, and $h^k_{eva}$ are the fluid density, mass flow rate, and enthalpy in region $k$, respectively.  $D^k_{i, eva}$ denote the inner hydraulic diameter in region $k$, $T^k_{w, eva}$ and $\bar{T}^k_{r, eva}$ represent the wall temperature and mean fluid temperature in region $k$, while $z$ and $t$ are defined as the length coordinate and time.   $\alpha^k_{i, eva}$ is the inner heat transfer coefficient in region $k$, which is obtained based on the local mass flow rates and phase distribution in each region. Specifically, the coefficients for the sub-cooled (\(k=1\)) and superheated (\(k=3\)) regions are scaled from reference values: 
\begin{subequations}
\begin{align}
\alpha^1_{i,\mathrm{eva}} &= \alpha^1_{i0,eva} \left(\frac{\dot{m}_{i,\mathrm{eva}}}{\dot{m}_{f_0}}\right)^{0.8}\\
\alpha^3_{i,\mathrm{eva}} &= \alpha^3_{i0,eva} \left(\frac{\dot{m}_{o,\mathrm{eva}}}{\dot{m}_{f_0}}\right)^{0.9}
\end{align}
\end{subequations}
where $\dot{m}_{i, eva}$ and $\dot{m}_{o,\mathrm{eva}}$ are the liquid and vapor mass flow rates, respectively; $\alpha^1_{i0,\mathrm{eva}}$ and $\alpha^3_{i0,\mathrm{eva}}$ denote the reference inner heat transfer coefficients in the sub-cooled and superheated regions, while $\dot{m}_{f0}$ is the reference mass flow rate. Meanwhile, $\alpha^2_{i,\mathrm{eva}}$ for the two-phase region is calculated as a semi-empirical function of the local phase fraction \(x_e\), the inner heat transfer coefficients of the neighboring regions, and the gas-to-liquid density ratio \(d = \rho_g / \rho_l\).
\begin{equation}
\alpha^2_{i,\mathrm{eva}} = \alpha^1_{i,\mathrm{eva}} 
\Bigg[ 
\big((1-x_e) + 1.2 x_e^{0.4} (1-x_e) d^{0.37}\big)^{-2.2} 
+ \bigg(\frac{\alpha^3_{i,\mathrm{eva}}}{\alpha^1_{i,\mathrm{eva}}} x_e^{0.01} (1 + 8(1-x_e)^{0.7}) d^{0.67}\bigg)^{-2} 
\Bigg]^{-0.5}
\end{equation}
In addition, the energy conservation of the pipe wall is also introduced to help increase the model accuracy:
\begin{equation}
(C_p\rho A)_{w,eva}\frac{\partial T^k_{w,eva}}{\partial t} = \pi(D^k_{o,eva}\alpha_{o,eva}(T^k_{a,eva}-T^k_{w,eva})+D^k_{i,eva}\alpha^k_{i,eva}(\bar{T}^k_{r,eva}-T^k_{w,eva}))
\end{equation}
where \((C_p \rho A)_{w, eva}\) represents the thermal capacitance of the evaporator pipe wall, defined as the product of the wall's specific heat at constant pressure \(C_p\), density \(\rho\), and cross-sectional area \(A\). $D^k_{o,eva}$ and  $T^k_{a,eva}$  are the outer diameter of the evaporator pipe and heat source temperature in region $k$, respectively, while $\alpha_{o,eva}$ is the outer heat transfer coefficient, obtained by scaling the reference outer coefficient according to the exhaust air mass flow rate:
\begin{equation}
\alpha_{o,\mathrm{eva}} = \alpha_{o0} \left(\frac{\dot{m}_{a}}{\dot{m}_{a_0}}\right)^{0.5}
\end{equation}
where $\alpha_{o0}$ and \(\dot{m}_{a0}\) are the reference outer heat transfer coefficient and reference exhaust air mass flow rate, respectively. Then, the heat source outlet temperature in region $k$ is presented as:
\begin{equation}
T^k_{o,eva} = e^{-\phi}T^k_{i,eva} + (1-e^{-\phi}) T^k_{w,eva}
\end{equation}
\begin{equation}
\phi = \frac{A^k_{tr, eva}\alpha_{o, eva}}{\dot{m}_a C_a} 
\end{equation}
where $T^k_{o,eva}$ and $T^k_{i,\mathrm{eva}}$ denote the outlet and inlet temperature in region $k$, respectively. $A^k_{\mathrm{tr, eva}}$ is the heat transfer area between the waste heat and the pipe wall in region $k$, and $\dot{m}_a$ and $C_a$ are the mass flow rate and specific heat capacity of the heat source, respectively. Following that, the lengths of the sub-cooled and two-phase regions, $L^1_{eva}$ and $L^2_{eva}$, as well as the evaporator pressure $P_{eva}$ and superheat $SH$, are calculated by integrating Eqs.~(\ref{eq_Aeva}) and~(\ref{eq_hkeva}) with the Leibniz formula, as expressed in  Eq.~(\ref{eq_leibniz}):
\begin{equation}
\label{eq_leibniz}
\int_{z_1}^{z_2} \frac{\partial f(z, t)}{\partial t}\, dz = \frac{d}{dt}\int_{z_1}^{z_2}{\partial  f(z, t)}\, dz + f(z_1, t) \frac{dz_1}{dt} - f(z_2, t) \frac{dz_2}{dt}
\end{equation}
and obtaining the intermediate physical properties (including evaporation pressure at the two-phase region and the saturation temperature at the evaporator outlet) through Refprop 9.1, thereby capturing the system dynamics.  On the basis of the above-mentioned relations, the first-principles  model of the ORC system is constructed as a seven-state ODE:
\begin{equation}
D_{eva}\dot{x}_{eva}=f_{eva}(x_{eva}, u_{eva})
\end{equation}
where $x_{eva}$ denotes the selected state variables [$L^1_{eva}$, $L^2_{eva}$ $T^1_{w,eva}$, $T^2_{w,eva}$, $T^3_{w,eva}$, $P_{eva}$, $SH$], while $u_{eva}$ denotes the input control actions [$\omega_{exp}$, $\omega_{pump}$]. Therefore, the inputs of the ORC system in this study consist of [$\omega_{exp,t}$, $\omega_{pump,t}$, $L^1_{eva,t}$, $L^2_{eva,t}$, $T^1_{w,eva,t}$, $T^2_{w,eva,t}$, $T^3_{w,eva,t}$, $P_{eva,t}$, $SH_t$] at the current time step $t$, while the outputs [$L^1_{eva,t+1}$, $L^2_{eva,t+1}$ $T^1_{w,eva,t+1}$, $T^2_{w,eva,t+1}$, $T^3_{w,eva,t+1}$, $P_{eva,t+1}$, $SH_{t+1}$] at time step $t+1$ entail the state trajectory  over the subsequent $n$ time steps, which is set to $5$ in this case. 

\begin{table}[htbp]
\centering
\caption{Parameter values for the ORC model}
\resizebox{0.8\linewidth}{!} {
\renewcommand{\arraystretch}{1.5} 
\begin{tabular}{>{\centering\arraybackslash}p{3.5cm}|>{\centering\arraybackslash}p{3.5cm}|>{\centering\arraybackslash}p{3.5cm}|>{\centering\arraybackslash}p{3.5cm}}
\hline
\textbf{Parameter} & \textbf{Value} & \textbf{Parameter} & \textbf{Value} \\
\hline
$V_{\mathrm{exp}}$ & $1.4 \times 10^{-4}~\mathrm{m^3}$ & $\eta_{v,\mathrm{exp}}$ & $0.6$ \\
$V_{\mathrm{pump}}$ & $1.3 \times 10^{-5}~\mathrm{m^3}$ & $\eta_{v,\mathrm{pump}}$ & $0.6$ \\
$A_{\mathrm{eva}}$ & $3.142 \times 10^{-4}~\mathrm{m^2}$ & $D_{i,\mathrm{eva}}$ & $0.02~\mathrm{m}$ \\
$D_{o,\mathrm{eva}}$ & $0.022~\mathrm{m}$ & $A_{w,\mathrm{eva}}$ & $6.91 \times 10^{-5}~\mathrm{m^2}$ \\
$\alpha_{o0,\mathrm{eva}}$ & $435~\mathrm{W/(m^2 \cdot K)}$ & $\alpha^{1}_{i0,\mathrm{eva}}$ & $1090.0~\mathrm{W/(m^2 \cdot K)}$ \\
$\alpha^{3}_{i0,\mathrm{eva}}$ & $661.2~\mathrm{W/(m^2 \cdot K)}$ & $\dot{m}_{f_0}$ & $0.132~\mathrm{kg/s}$ \\
$\dot{m}_{a_0}$ & $0.200~\mathrm{kg/s}$ & $C_{p,w,\mathrm{eva}}$ & $0.385~\mathrm{kJ/(kg \cdot K)}$ \\
$\rho_{w,\mathrm{eva}}$ & $8.96 \times 10^{3}~\mathrm{kg/m^3}$ & $L^1_{eva}+L^2_{eva}+L^3_{eva}$ & $9~\mathrm{m}$ \\
\hline
\end{tabular}
}
\label{tab:orc_params}
\end{table}
Similarly, open-loop simulations of this ORC system are also performed using the explicit Euler method to generate data that mimics the real-world data for neural network training. The detailed system parameters are summarized in Table~\ref{tab:orc_params}.

\subsection{System Modeling}
To further evaluate the model's robustness against noise in the application of the ORC system, Gaussian noise is also added to the training data (i.e., output trajectory $\v y$) in an additive manner following Eq.~(\ref{eq_Gaussian}).
Meanwhile, all models are trained and evaluated using consistent configurations, including MinMax scaling (scaling the data to a range between 0 and 1), a batch size of 256, a two-layer neural architecture with 64 neurons per layer (i.e., model size of (64, 64)), the mean-squared error (MSE) loss function, and the Adam optimizer. Detailed hyperparameter values, including the FLOPs for quantifying model complexity, are provided in Table~\ref{tab_hyper_energy}. Based on FLOPs, the model complexity ranks as follows: RNN $<$ ICL-RNN $<$ ICRNN $<$ LRNN $<$ LSTM. Furthermore, it is important to note that all experiments were conducted three times. In the plots, the solid line represents the mean across the three trials on different unseen reactions, while the shaded region illustrates the standard deviation.

\begin{table}[htbp]
\centering
\caption{Hyperparameters of neural network models}
\vspace{.5em}
\resizebox{0.8\linewidth}{!} {
\renewcommand{\arraystretch}{1.5} 
\begin{tabular}{>{\centering\arraybackslash}p{3.5cm}|>{\centering\arraybackslash}p{3.5cm}|>{\centering\arraybackslash}p{3.5cm}|>{\centering\arraybackslash}p{3.5cm}}
\hline
\textbf{Model} & \textbf{Activation} & \textbf{Inputs} & \textbf{FLOPs}\\
\hline
Plain RNN & Tanh & $\vt x_t$ & 30, 499\\
Plain LSTM & Tanh & $\vt x_t$ & 260, 736\\
LRNN & Tanh & $\vt x_t$ & 162, 083\\
ICRNN & ReLU & $\vhat x_t=\bmat{\vt x_t,-\vt x_t}^\T$ & 82, 979\\
ICL-RNN (ours)& ReLU & $\vhat x_t=\bmat{\vt x_t,-\vt x_t}^\T$ & 30, 755\\
\hline
\end{tabular}}
\label{tab_hyper_energy}
\end{table}

The modeling results are comparable to those observed in the CSTR case, including the analysis and conclusions. From Fig.~\ref{fig_energy_loss}, for the testing MSE, the $y$-axis of the left figure is presented in a linear scale, while the right figure uses a logarithmic scale. All three models - ICRNN, LRNN, and ICL-RNN - demonstrate robustness to noise. However, in terms of modeling accuracy, the performance ranks as LRNN $>$ ICL-RNN $>$ ICRNN. Nonetheless, the advantage of the input convex property becomes evident in controlling neural network-based MPC.

\begin{figure}[ht!]
    \centering
    \begin{subfigure}[t]{0.48\textwidth}
        \centering
        \includegraphics[width=\columnwidth]{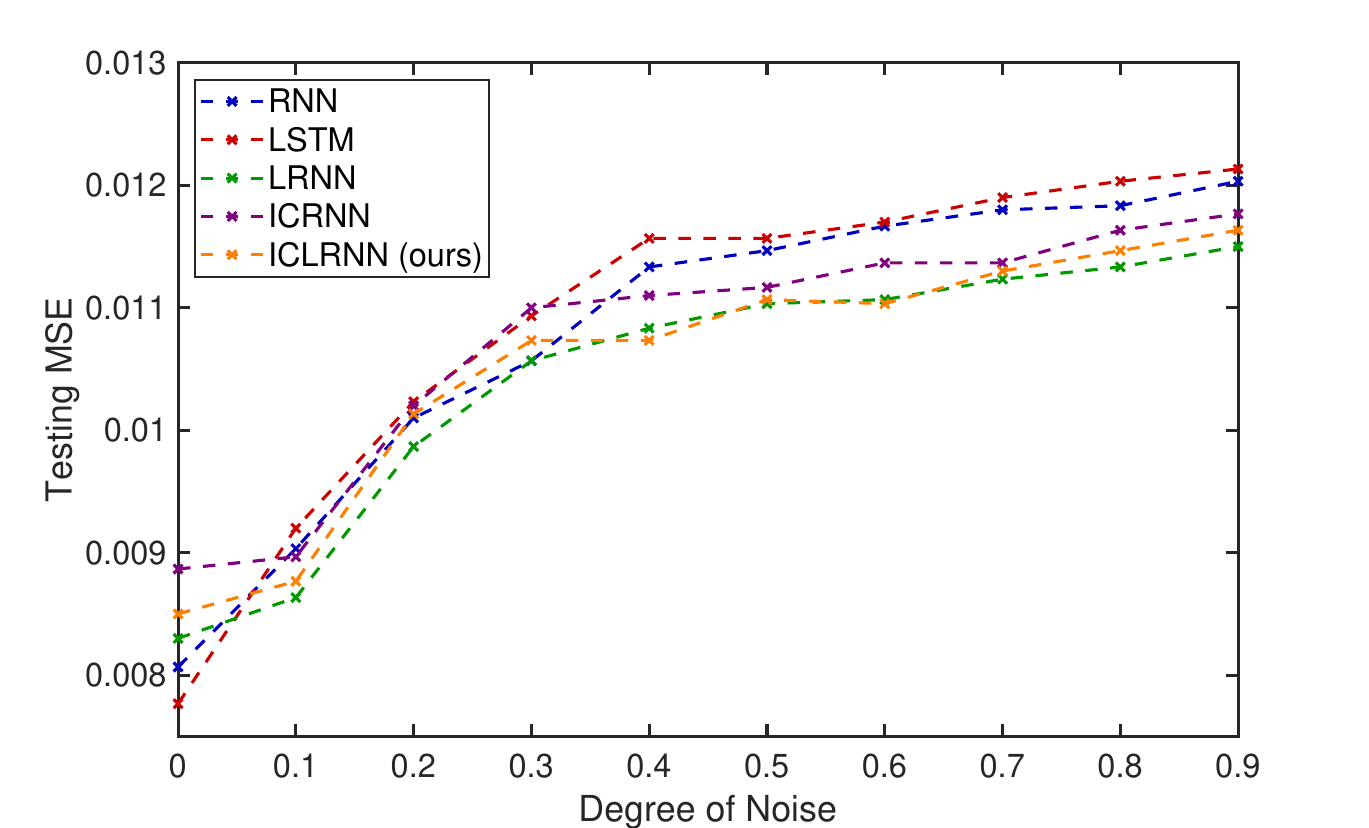} 
        \caption{Linear scale}
        \label{fig_linear_loss_energy}
    \end{subfigure}
    ~ 
    \begin{subfigure}[t]{0.48\textwidth}
        \centering
        \includegraphics[width=\columnwidth]{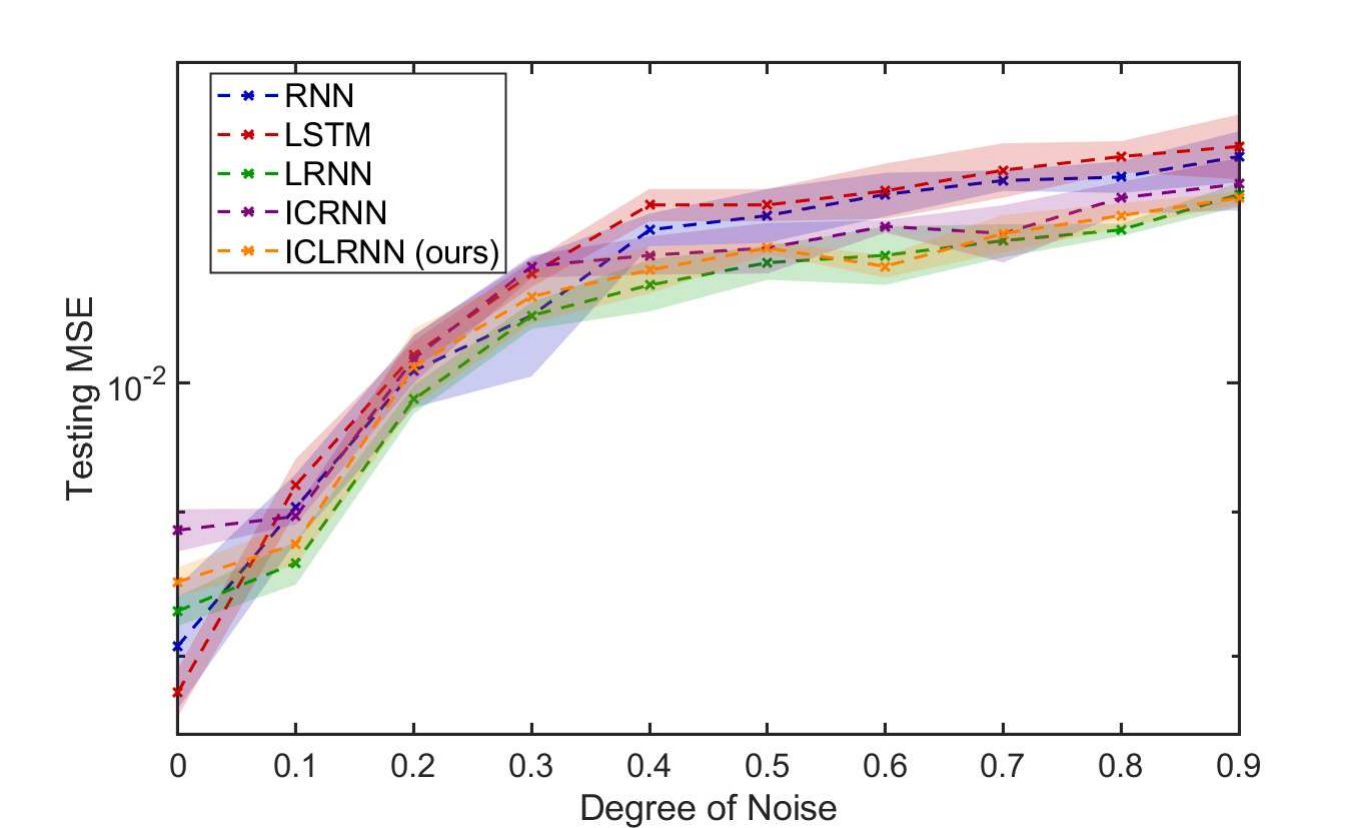}
        \caption{Logarithmic scale}
        \label{fig_log_loss_energy}
    \end{subfigure}
    \caption{Testing MSE vs. degree of additive noise over 3 random trials.}
    \label{fig_energy_loss}
\end{figure}

\begin{figure}[ht!]
    \centering
    \begin{subfigure}[t]{0.48\textwidth}
        \centering
        \includegraphics[width=\columnwidth]{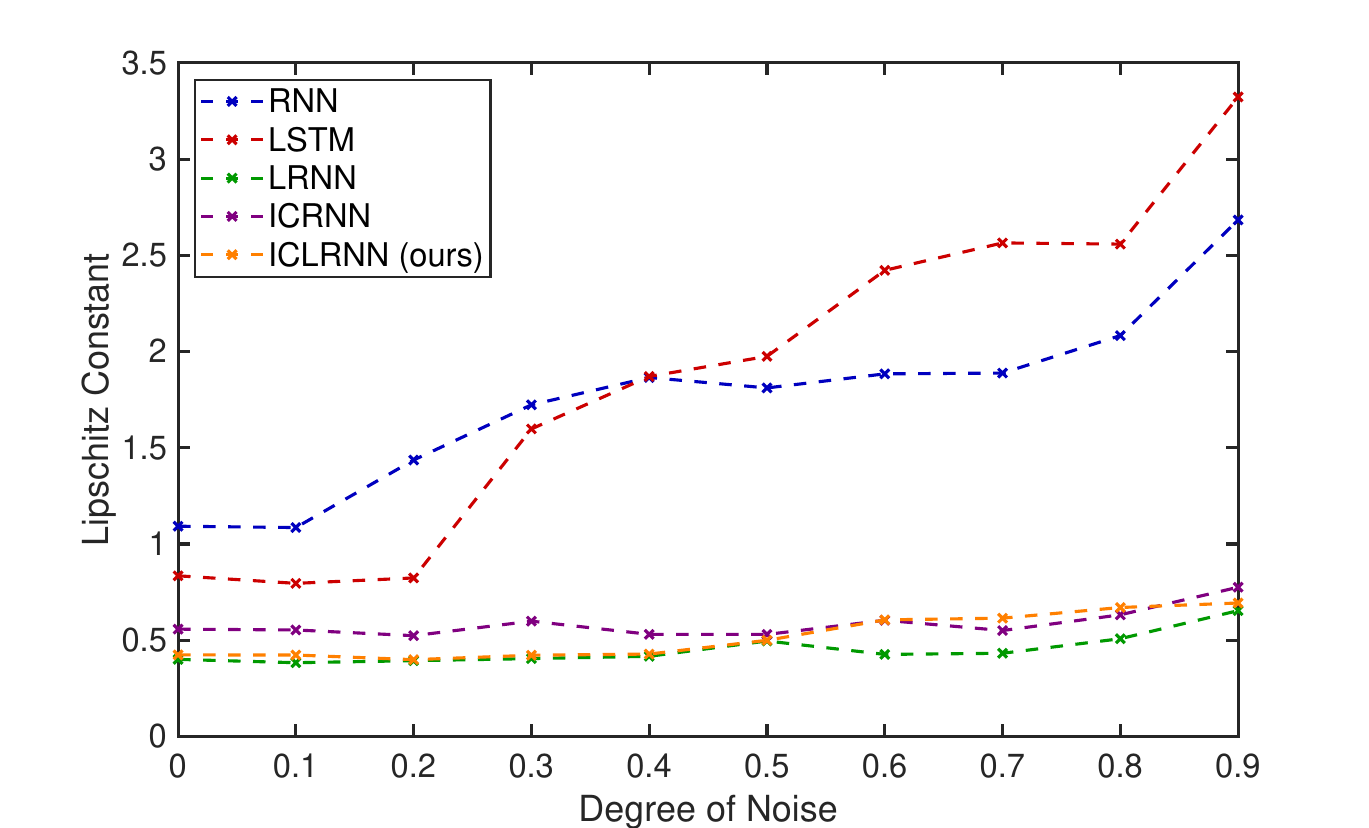}
        \caption{Linear scale}
        \label{fig_linear_lip_energy}
    \end{subfigure}
    ~ 
    \begin{subfigure}[t]{0.48\textwidth}
        \centering
        \includegraphics[width=\columnwidth]{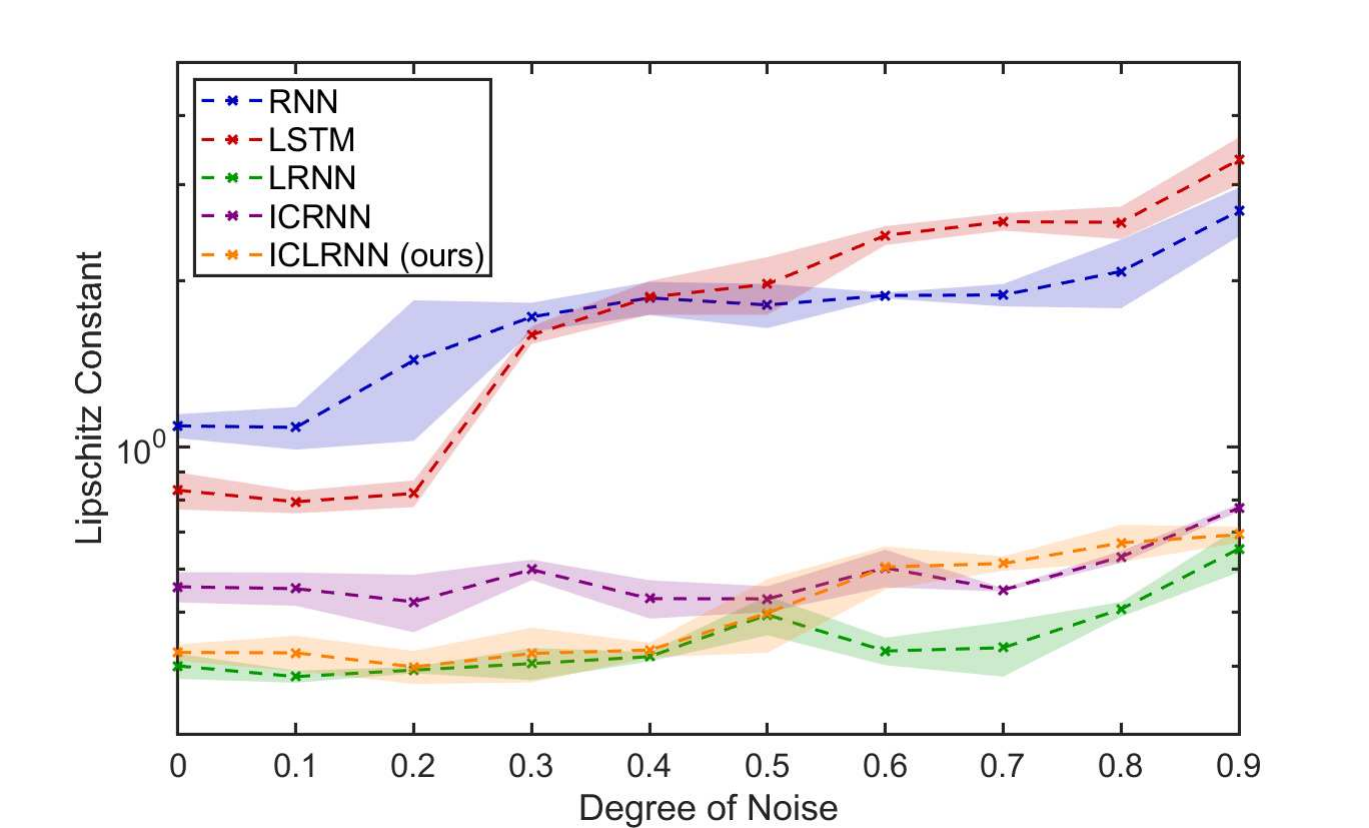}
        \caption{Logarithmic scale}
        \label{fig_log_lip_energy}
    \end{subfigure}
    \caption{Lipschitz constant of the trained model vs. degree of additive noise over 3 random trials.}
    \label{fig_energy_lip}
\end{figure}

Similarly, from Fig.~\ref{fig_energy_lip}, for the Lipschitz constant, the $y$-axis of the left figure is presented in a linear scale, while the right figure uses a logarithmic scale. Due to the recurrent nature of the neural networks, obtaining a precise Lipschitz constant is challenging. Therefore, a numerical method was used to approximate it (see Remark~\ref{remark_numerical}). For all three models - ICRNN, LRNN, and ICL-RNN - the Lipschitz constant is less than or equal to 1, with values that are approximately the same, while any discrepancies observed are due to approximation errors. 

The ICRNN model becomes untrainable when the number of neurons per layer exceeds 256. In contrast, our ICL-RNN remains trainable, as the Lipschitz constraint stabilizes the training process (see Table~\ref{tab_nan_energy}). This stabilization property, achieved through spectral normalization, has been previously observed and utilized in stabilizing the training of GANs in generative AI \citep{miyato2018spectral}, but it has not yet been widely discussed in the context of RNNs.
\begin{table}[htbp]
\centering
\vspace{.5em}
\caption{Testing MSE of ICRNN and ICL-RNN with respect to increasing hypothesis space}
\label{tab_nan_energy}
\resizebox{1\linewidth}{!} {
\renewcommand{\arraystretch}{1.5} 
\begin{tabular}{c|c|c|c|c}
\hline
& \multicolumn{2}{c|}{\textbf{Testing MSE}} & \multicolumn{2}{c}{\textbf{Lipschitz Constant}}\\
\cline{2-5}
\textbf{Model Size} & \textbf{ICRNN} & \textbf{ICL-RNN} & \textbf{ICRNN} & \textbf{ICL-RNN} \\
\hline
(32, 32) & $8.913 \times 10^{-3} \pm 1.14 \times 10^{-4}$& $8.697 \times 10^{-3} \pm 8.81 \times 10^{-5}$& $5.03 \times 10^{-1} \pm 1.03 \times 10^{-1}$ & $5.03 \times 10^{-1} \pm 8.50 \times 10^{-2}$\\
(64, 64) & $8.926 \times 10^{-3} \pm 7.83 \times 10^{-4}$& $8.636 \times 10^{-3} \pm 9.13 \times 10^{-5}$& $6.62 \times 10^{-1} \pm 1.85 \times 10^{-1}$& $5.11 \times 10^{-1} \pm 8.41 \times 10^{-2}$\\
(128, 128) & $1.944 \times 10^{-2} \pm 2.69 \times 10^{-4}$& $8.612 \times 10^{-3} \pm 9.22 \times 10^{-5}$& $7.92 \times 10^{-1} \pm 5.03 \times 10^{-2}$ & $5.30 \times 10^{-1} \pm 8.59 \times 10^{-2}$\\
(256, 256) & $2.304  \pm 2.29 \times 10^{-2}$& $8.904 \times 10^{-3} \pm 1.13 \times 10^{-5}$& $3.352 \pm 1.262$ & $4.07 \times 10^{-1} \pm 5.40 \times 10^{-2}$\\
(512, 512) & NaN& $8.046 \times 10^{-3} \pm 7.79 \times 10^{-6}$& NaN& $4.90 \times 10^{-1} \pm 1.60 \times 10^{-2}$\\
\hline
\end{tabular}
}
\end{table}
\subsection{Computational Efficiency in MPC for ORC Control}
In this section, we evaluate the computational efficiency of MPC using different models for the ORC system. This case study represents a more complex and practically representative benchmark, allowing us to further validate the computational efficiency enhancement achieved by the proposed ICLRNN and examine whether these advantages generalize to systems with different dynamics characteristics. Specifically, the initial states of the ORC system are  shown in Table~\ref{tab:eva_orcinitialstates}, among which $SH$ and $P_{eva}$ are the control objectives with the setpoints, $SH_S$ and $P_{eva,s}$, set to 14.83 K and 2287.03 kPa, respectively. 
\begin{table}[htbp]
\centering
\caption{Initial States of the ORC system}
\resizebox{0.8\linewidth}{!} {
\renewcommand{\arraystretch}{1.5} 
\begin{tabular}{>{\centering\arraybackslash}p{3.5cm}|>{\centering\arraybackslash}p{3.5cm}|>{\centering\arraybackslash}p{3.5cm}|>{\centering\arraybackslash}p{3.5cm}}
\hline
\textbf{Initial State} & \textbf{Value} & \textbf{Initial State} & \textbf{Value} \\ \hline
$T^{1}_{w,\mathrm{eva}}$ & $424.78~\mathrm{K}$ & $T^{2}_{w,\mathrm{eva}}$ & $433.23~\mathrm{K}$ \\ 
$T^{3}_{w,\mathrm{eva}}$ & $474.75~\mathrm{K}$ & $L^{1}_{\mathrm{eva}}$ & $4.14~\mathrm{m}$ \\ 
$L^{2}_{\mathrm{eva}}$ & $3.47~\mathrm{m}$ & $P_{\mathrm{eva}}$ & $2330.00~\mathrm{kPa}$ \\ 
$SH$ & $16.97~\mathrm{K}$ & & \\ \hline
\end{tabular}}
\label{tab:eva_orcinitialstates}
\end{table}


Fig.~\ref{fig_ORC_results} shows the Closed-loop state profiles under the MPCs based on RNN, LSTM, LRNN, ICRNN, and ICLRNN for the ORC system. Consistent with the results obtained for the CSTR control case, both ICLRNN-based and ICRNN-based MPC also achieve shorter computation time for the ORC system compared with RNN, LSTM, and LRNN. In particular, the average MPC computation time for the optimal solution with each model is illustrated in Table~\ref{tab_mpctimeORC}.  

\begin{figure}[ht!]
    \centering
    \begin{subfigure}[t]{0.49\textwidth}
        \centering
        \includegraphics[width=\textwidth]{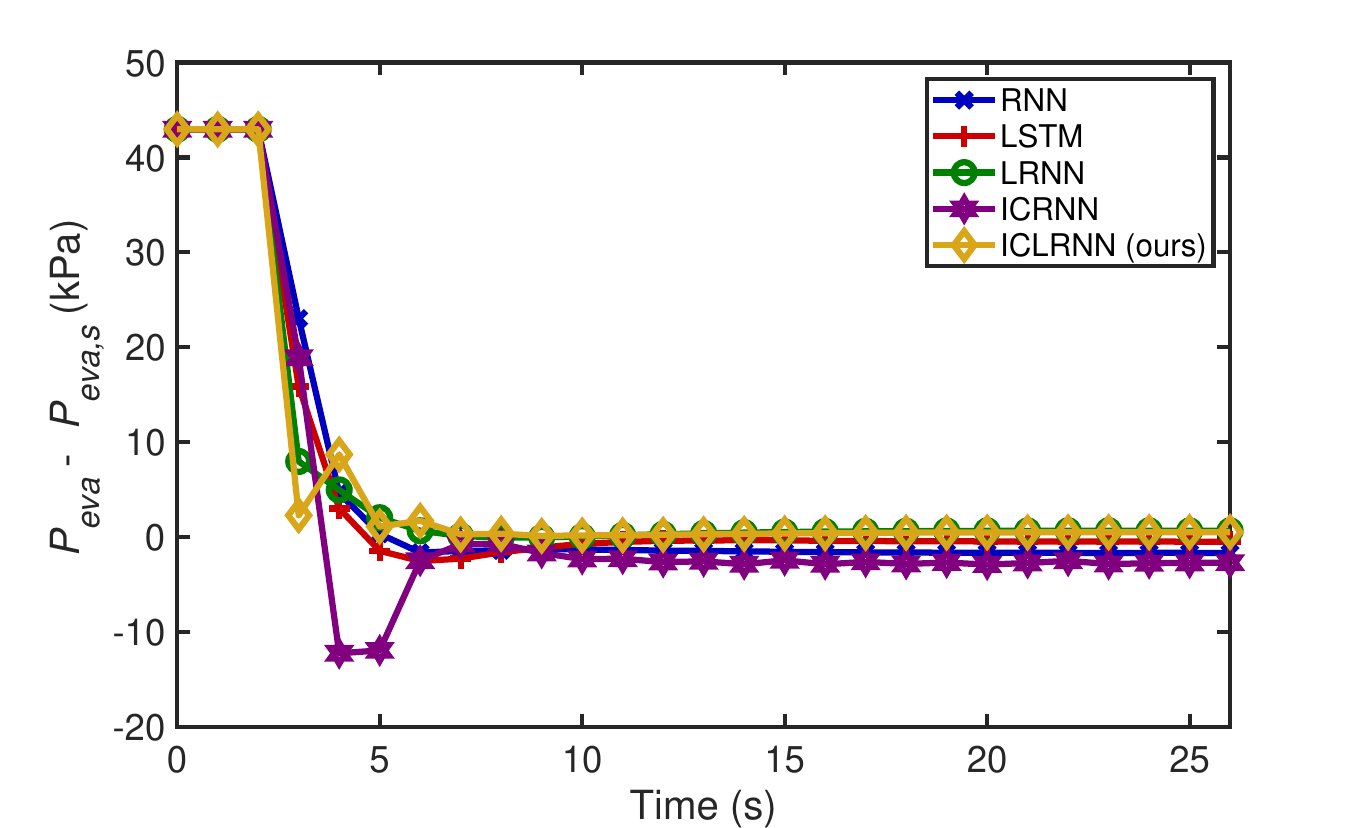}
        \label{fig_ORC_Pe}
    \end{subfigure}
    \hspace{0.001\textwidth}
    \begin{subfigure}[t]{0.49\textwidth}
        \centering
        \includegraphics[width=\textwidth]{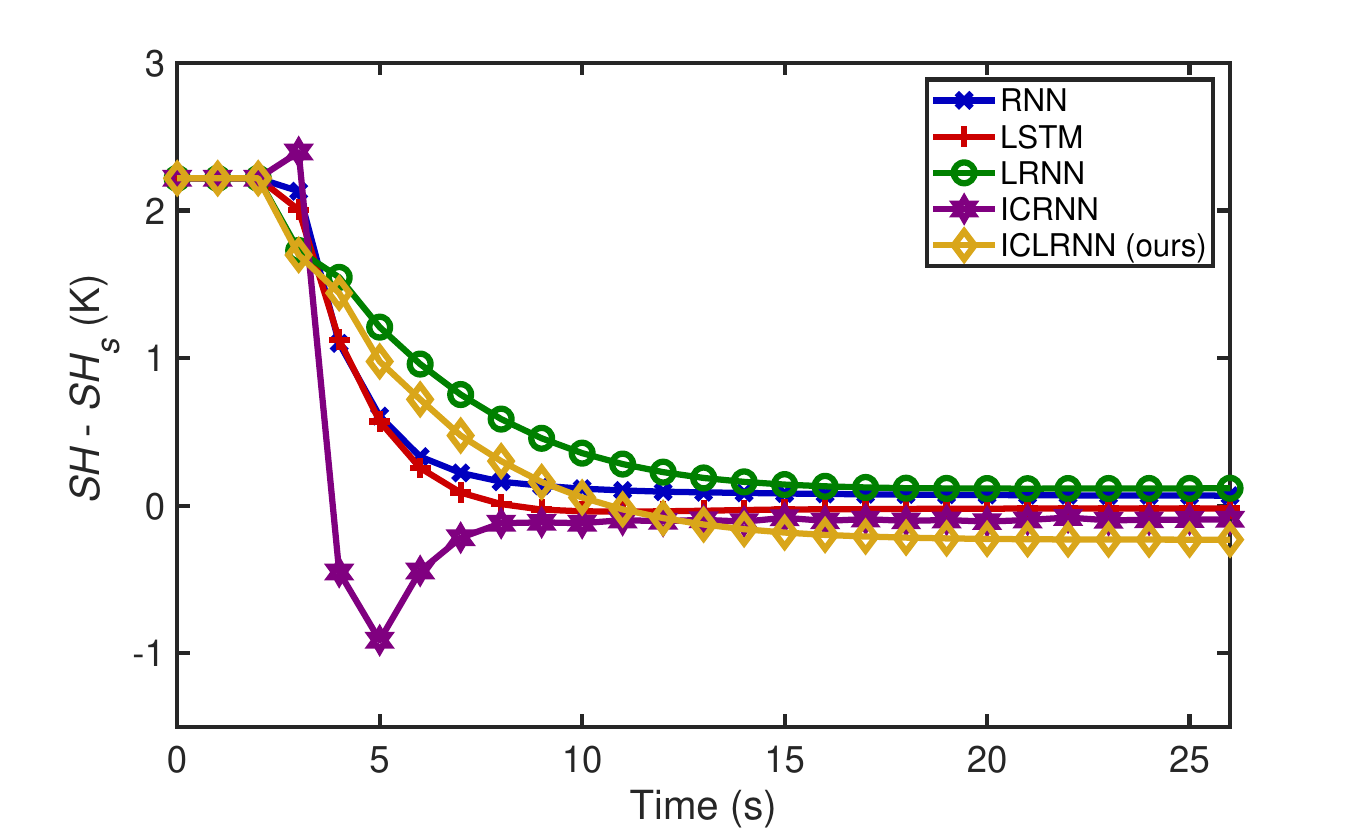}
        \label{fig_ORC_SH}
    \end{subfigure}

    \caption{Closed-loop state profiles under neural network model-based MPC}
    \label{fig_ORC_results}
\end{figure}

\begin{table}[htbp]
\centering
\renewcommand{\arraystretch}{1.5} 
\caption{MPC computation time of neural network models}
\vspace{.5em}
\resizebox{0.8\linewidth}{!} {
\begin{tabular}{>{\centering\arraybackslash}p{4cm}|>{\centering\arraybackslash}p{10cm}}
\hline
\textbf{Model} & \textbf{Average computation time (s)}\\
\hline
Plain RNN & 174.195
\\
Plain LSTM & 170.488\\
LRNN & 164.395\\
ICRNN & 133.237\\
ICL-RNN (ours) & 137.672\\
\hline
\end{tabular}}
\label{tab_mpctimeORC}
\end{table}

Specifically, ICLRNN-MPC reduces the average computation time by 20.97\%, 19.25\%, and 16.26\% compared to RNN-MPC, LRNN-MPC, and LSTM-MPC, respectively, while is 3.33\% higher than ICRNN-MPC. This superior computational efficiency of ICLRNN and ICRNN further confirms the effectiveness of the introduced convex optimization and highlights their advantage in real-time applicability. Furthermore, since the Lipschitz constraint significantly improves the ICLRNN’s robustness over ICRNN in the ORC modeling study, this careful integration of convexification and gradient restriction serves as a comprehensive development, which enables fast convergence of the optimal solution while preserving robustness, even in the presence of the ORC system with more complexity and stronger nonlinearity. 

In general, the proposed ICLRNN outperforms conventional recurrent units in both computational efficiency and robustness through carefully incorporating the convex optimization and the Lipschitz constraint. Its successful application to the modeling and control of engineering scenarios, including the CSTR system and ORC system, demonstrates not only strong noise tolerance but also consistently high computational efficiency during optimal solution. These combined attributes make ICLRNN particularly well-suited for real-time industrial deployment, where both robustness to data uncertainty and rapid solution of optimal control problems are critical for smooth operation.

\section{Conclusion}
This work developed an Input Convex Lipschitz RNN to improve computational efficiency and robustness of neural network-based system modeling and control in various engineering applications, such as the modeling and control of chemical processes and energy systems. The proposed ICL-RNN addresses the challenge of integrating input convexity and Lipschitz continuity, which are often at odds: enforcing one property in neural network design can compromise the other. To overcome this, rigorous analysis was conducted to ensure that both Lipschitz continuity and input convexity are simultaneously satisfied. Finally, the ICL-RNN was shown to outperform state-of-the-art recurrent units in the modeling and control of chemical and energy systems.
\section{Acknowledgments}
Financial support from NRF-CRP Grant 27-2021-0001 is gratefully acknowledged.

\newpage
\bibliography{reference}

\end{document}